\documentclass[10pt,twocolumn,letterpaper]{article}

\usepackage{iccv}
\usepackage{times}
\usepackage{epsfig}
\usepackage{graphicx}
\usepackage{amsmath}
\usepackage{amssymb}
\usepackage{arydshln}
\usepackage{color}
\usepackage{multirow}
\usepackage{microtype}
\usepackage{subfigure}
\usepackage{booktabs}
\usepackage{amsthm}

\newtheorem{theorem}{Theorem}
\usepackage{algorithm}  
\usepackage{algorithmic} 
\usepackage[breaklinks=true,bookmarks=false]{hyperref}

\iccvfinalcopy % *** Uncomment this line for the final submission

\def\iccvPaperID{****} % *** Enter the ICCV Paper ID here

\begin{document}

%%%%%%%%% TITLE
\title{FAT: Learning Low-Bitwidth Parametric Representation \\via Frequency-Aware Transformation}

\author{Chaofan Tao$^1$, \quad Rui Lin$^1$,\quad Quan Chen$^2$,\quad Zhaoyang Zhang$^3$,\quad Ping Luo$^1$,\quad Ngai Wong$^1$\\
$^1$The University of Hong Kong \quad
$^2$SUSTech \quad
$^3$The Chinese University of Hong Kong


}

\maketitle
%\thispagestyle{empty}

%%%%%%%%% ABSTRACT
\begin{abstract}
Learning convolutional neural networks (CNNs) with low bitwidth is challenging because performance may drop significantly after quantization. 
Prior arts often discretize the network weights  by carefully tuning  hyper-parameters of quantization (\eg non-uniform stepsize and layer-wise bitwidths),
% to minimize quantization error, which is the discrepancy between full- and low-precision . 
which are complicated and sub-optimal because the full-precision and low-precision models have large discrepancy.
%
%the fundamentally different solution spaces. 
This work presents a novel quantization pipeline, Frequency-Aware Transformation (FAT), which has several appealing benefits.
%
%a substantially simpler approach called , which has  
(1) Rather than designing complicated quantizers like existing works, FAT learns to transform network weights in the frequency domain before quantization, making them more amenable to training in low bitwidth.
(2) With FAT, CNNs can be easily trained in low precision using simple standard quantizers without tedious hyper-parameter tuning. Theoretical analysis shows that FAT improves both uniform and non-uniform quantizers.
(3) FAT can be easily plugged into many CNN architectures. When training ResNet-18 and MobileNet-V2 in 4 bits, FAT plus a simple rounding operation\footnote{A rounding operation uniformly quantizes a continuous value to its nearest discrete value.} already achieves 70.5\% and 69.2\% top-1 accuracy on ImageNet without bells and whistles, outperforming recent state-of-the-art by reducing 54.9$\times$ and 45.7$\times$ computations against full-precision models.  We hope FAT provide a novel perspective for model quantization. Code is available at \textcolor{blue}{\url{https://github.com/ChaofanTao/FAT_Quantization}}.
\end{abstract}
\section{Introduction}
\label{sec:intro}
% why compression
Convolutional neural networks (CNNs) have exhibited impressive capabilities in various real-world applications. However, such amazing performance usually comes at the expense of a humongous amount of storage and computation. For example, ResNet-101~\cite{he2016deep} entails 44.6M parameters and over $7.9\times10^3$M multiply-accumulate (MAC) operations per image in ImageNet~\cite{russakovsky2015imagenet}. These requirements hinder mobile applications (e.g., on cell phones and robots) where devices have restrictive storage, power and computing resources. In order to compress CNNs, pruning \cite{kang2020operation, lin2020hrank} and distillation \cite{tan2018distill, sun2019patient} attempt to learn a compact model with fewer weights than the original model. Differently, quantization methods shrink the bitwidth of data by replacing float values with finite-bitwidth integers, thus reducing memory footprint and simplifying computational operations. 

% drawback of current quantization-aware training 
Previous approaches \cite{zhang2018lq,li2020additive, habi2020hmq, wang2020differentiable, esser2020learned} generally model the task of quantization as an error minimization problem, viz. $\min$ $ \left \| \mathcal{W} - \rm Q(\mathcal{W}) \right \|,$ where $\mathcal{W}$ is the weight and $\rm Q(\cdot)$ the quantizer. As shown in Figure~\ref{fig:intro}, to minimize the quantization error, complicated quantizers are trained involving mixed-precision across layers or channels, adaptive quantization levels or learnable training policy, e.g. reinforcement learning-based policy.  It has three potential problems:  1) The solution spaces of full-precision and quantized models are quite different (continuous vs discrete), especially for low bitwidths quantized models. The quantized model has very limited capacity to represent its weights. Simply pushing full-precision values to their quantized representations is sub-optimal. Ref.~\cite{qin2020forward} shows that flipping the signs of weights moderately during binarization leads to better performance compared with vanilla binarization; 2) Weights are correlated with each other in each CNN filter, while previous methods quantize each weight independently without exploring such relationship; 3) The quantized weight $\mathcal{W}_q$ has a null gradient $\frac{\partial \mathcal{W}_q}{\partial \mathcal{W}} $ versus their float counterparts. Straight-through estimator (STE)~\cite{bengio2013estimating} and its soft variant DSQ \cite{gong2019differentiable} are heuristics to approximate gradients, yet how to estimate accurate gradients is still an open challenge.

This work poses a question: \textit{Instead of learning complicated quantizers to fit the full-precision parameters, can we generate quantization-friendly representations to fit a basic quantizer?} We propose to decompose quantization as a representation transform ${\rm T}(\cdot)$ and a standard quantizer ${\rm Q}(\cdot)$ such as a uniform quantizer. Before sending to the quantizer, the weight parameters are transformed $\mathcal{W}_{t} = \rm T(\mathcal{W})$. The transform works to bridge the capacity gap between full-precision and low-bitwidth model by suppressing redundant information and retaining useful ones. By carefully defining the transform, we can jointly quantize parameters by exploring the relationships among neurons and obtain informative gradients.

In practice, we use the transform $\rm T(\cdot)$ to map the original weights to the frequency domain, emphasize important frequency components and then map the weights back to the spatial domain.  Powered by Discrete Fourier Transform (DFT), any element in the frequency domain associate all elements of weights in the spatial domain. Therefore, the transform analyses the weights holistically in the frequency domain, rather than treating each parameter separately. By learning a mask over the frequency map of weights, the transform selectively retains informative frequencies, while  masking off trivial frequencies from flowing into a restrictive low-bitwidth model. After then, a standard quantizer (e.g., uniform or logarithmic quantizer) is used to quantize parameters to a prescribed bitwidth. 

In backpropagation, the discretization gradient $\frac{\partial \mathcal{W}_q}{\partial \mathcal{W}} = \frac{\partial \mathcal{W}_q}{\partial \rm T(\mathcal{W})} \frac{\partial \rm T(\mathcal{W})}{\partial \mathcal{W}}$ becomes controllable by the explicitly defined transform. For simplicity, we employ the same quantizer for both  weight and activation, while not transforming activation since activation does not relate to the capacity of the neural network. After training, the weight transform is removed. Deployment becomes as simple as using a standard quantizer. Hence, no special hardware~\cite{krishnamoorthi2018quantizing} that supports advanced quantizers (e.g., mixed-precision) are required.

%As shown in Figure~\ref{fig:intro}, an intuitive explanation why the transform work is visualized by the distribution of original weight (blue) and transformed weight (orange). By deactivating the trivial  in the weight, the transformed weight tends to tighten towards zero. Therefore, transformed weights has much fewer proportion to be treated as outliers and be clipped. 

% why using Fourier Transform in the T()
% global view, explaniable 
% + visualization + observation

% Analog-to-Digital Converter (ADC) is an important module in mixed-signal processing, whereby continuous-time signal is discretized into discrete-time digital signals to enjoy the advanced digital signal processing (DSP) brought forth by sophisticated digital processors.  It can be generally divided into two steps, sampling representative features and then quantization. Inspired by this, ... \\

% \textit{3) Why the denoised weight works better than original weights in the quantization? }
% Every bits matters in the quantization. \textbf{The transformed denoised weight takes full advantages of the capacity of each bit}, by preserving informative weight in the model via a suitable transform. In our model,  \textbf{the quantizer also feedback quantization performance to adjust this transform.}

\begin{figure}
\centering
\includegraphics[scale=0.23]{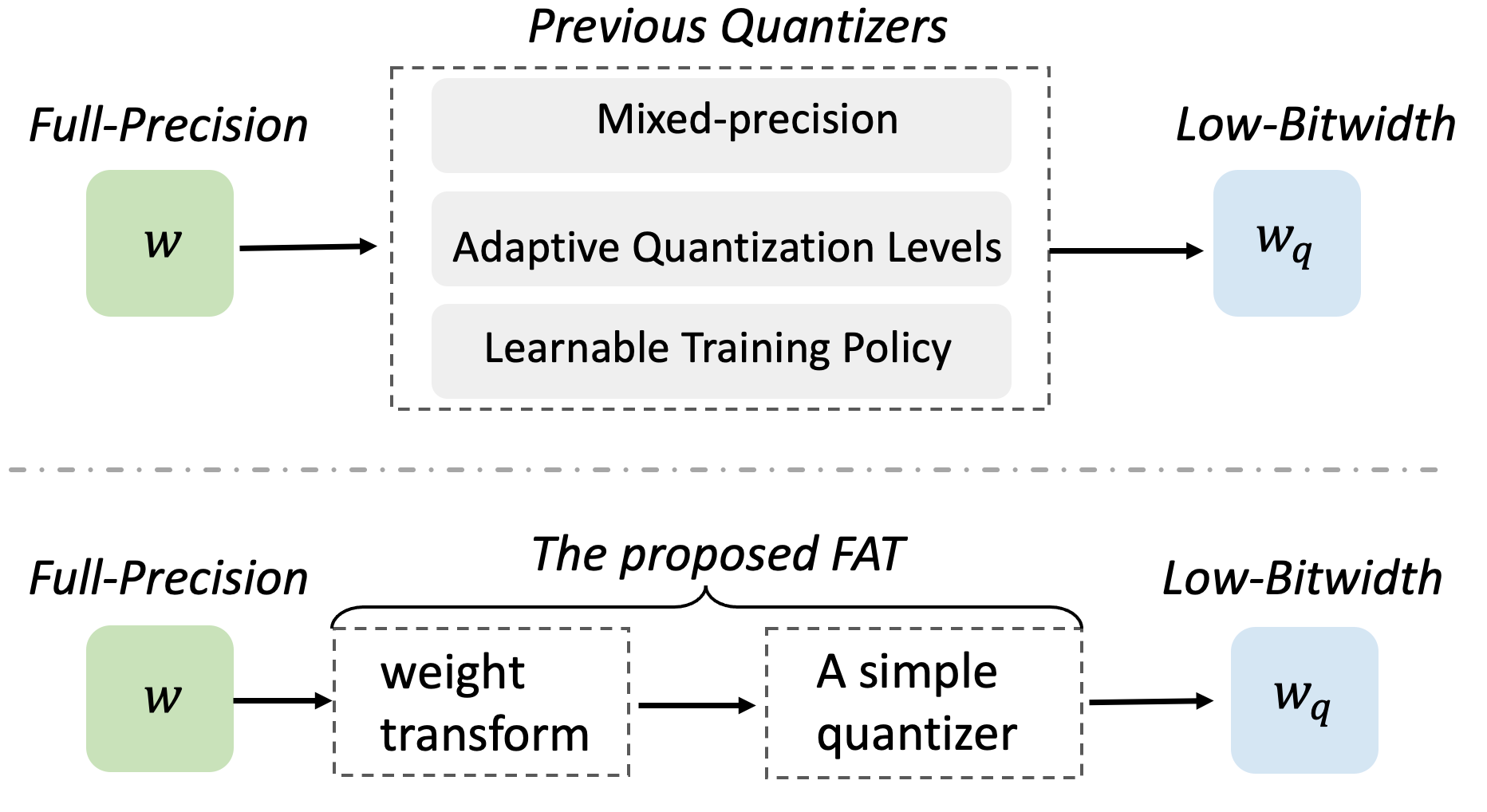}
\caption{Comparison between previous quantization methods (top) with FAT (bottom). The proposed method decouples a quantization problem with a transform and a standard quantizer. The transform suppresses trivial components while keeping informative ones, and learns the quantization gradient by exploiting relationships among neurons. FAT does not involve mixed-precision, adaptive quantization levels or learnable training policy. During inference, only a simple and standard quantizer is required.}
\label{fig:intro}
\end{figure}

We empirically observe that the learned mask prevent most of high-frequency components from flowing to the low-bit model. We further provide theoretical analyses about the properties of our model from a frequency perspective. The proposed Frequency-Aware Transformation (FAT) enables not only quantization error reduction, but also informative discretization gradient by jointly considering multiple frequencies. The main contributions of this work are fourfold:
\begin{enumerate}
	\item To the best of our knowledge, this is the first work that models the task of quantization via a representation transform and a standard quantizer. The proposed transform is an easy drop-in. We combine the transform with uniform/logarithmic quantizer in this paper.
	\item Powered by Fourier Transform, we introduce a novel spectral transform to generate quantization-friendly representations. The discretization gradient is enriched by exploring relationships between neurons, rather than quantizing neurons separately. 
    \item We theoretically analyse properties of the proposed transform. It deepens our understanding of quantization from a frequency-domain viewpoint. 
    \item We outperform state-of-the-art methods on CIFAR-10 and ImageNet datasets with higher computation reduction, pushing both weight and activation to INT3 against performance of full-precision models. The model is also deployed on an ARM-based mobile board.
\end{enumerate}

%% Related work
\section{Related Work}
\label{sec:lit_review}
% Quantization: minimize quantization OR gradient error Quantization schemes can generally be categorised into two classes, namely, post-training quantization and quantization-aware training.
\subsection{Quantization Methods}
\label{sec:lit_quantization}
\textbf{Post-Training Quantization} Post-training quantization \cite{baskin2018uniq, banner2019post, nagel2019data, nagel2020up} needs no training but a subset of dataset for calibrating the quantization parameters, including the clipping threshold and bias correction. Commonly, the quantization parameters of both weight and activation are decided before inference. Currently, post-training quantization methods cannot achieve satisfying performance  when the allowed bitwidth goes smaller, since the post-training quantization errors are accumulated layer by layer \cite{stock2019and}. Instead, quantization-aware training enables the model to adapt itself to a low-bitwidth setting.

% For instance, famous quantization libraries like TensorRT \cite{migacz20178} and TF-Lite \cite{krishnamoorthi2018quantizing} only support 8-bit quantization under acceptable performance drop.

% ACIQ \cite{banner2019post} shows that deciding the quantization parameters during running-time could achieve better performance than static post-training quantization. However, dynamic post-training quantization brings remarkable difficulties for inference device design.

% from post-training quantization to quantization-aware training

% For the sake of memory saving, we use quantization techniques to limit network weights and internal activations in low bitwidths instead of full-precision data. Generally, quantization approaches can be divided into two categories according to how they work to maintain the performance of the full-precision models.

% type 1
\textbf{Quantization-Aware Training} Quantization-aware training \cite{choi2018pact, zhang2018lq, gong2019differentiable, qin2020forward} generally focuses on minimizing the gaps between the quantized parameters and the corresponding full-precision ones. In~\cite{rastegari2016xnor,zhou2016dorefa,mishra2017wrpn,faraone2018syq}, the scaling factors for quantized parameters make the approximation of full-precision parameters more accurate. In~\cite{wang2018two}, the weights and activations are quantized separately in a two-step strategy. Mixed-precision is widely employed to achieve smaller quantization errors, such as LQ-Net~\cite{zhang2018lq}, DJPQ~\cite{wang2020differentiable} and HMQ~\cite{habi2020hmq}. In HAQ~\cite{wang2019haq}, the training policy is learned by reinforcement learning. Given a bitwidth, adaptive non-uniform quantization intervals~\cite{caccia2020online} can also reduce the quantization errors. 

In addition, the non-differentiable quantization function leads to zero-gradient problem during training. Although STE~\cite{bengio2013estimating} can be employed, the approximation error is large when the bitwidth is low.  DSQ~\cite{gong2019differentiable} uses a series of hyperbolic tangent functions to gradually approach the staircase function. Despite this, the aforementioned methods process quantization of weights independently. Our proposed FAT tackles quantization with an explicit transform. It enables joint quantization from a holistic view during feed-forward, and informative gradient during backpropagation. 

% In~\cite{yin2019blended}, the mismatch problem is alleviated by updating the parameters using additional information from the full-precision weights and their quantized counterparts.  
% considers both of the two perspectives. FAT minimizes the quantization error in the forward pass by transfomring the weights to be quantization-friendly, and reduces the gradient error in the backward pass through a novelty approximated quantization function based on the masks and FFT operation in our framework.

% FFT Application in DNNs
\subsection{Learning in the Frequency Domain}
\label{sec:lit_FFT_application}
Standard convolution in the spatial domain is mathematically equivalent with the cheaper Hadamard product in the frequency domain. Therefore, Fast Fourier Transform (FFT) has been widely used to speed up convolution~\cite{mathieu2013fast,rippel2015spectral,lin2019fast}, and design energy-efficient CNN hardware~\cite{nguyen2016energy,chitsaz2020acceleration}. CNNpack~\cite{wang2016cnnpack} regards convolutional filters as images and then decomposes convolution in the frequency domain for speedup.

Moreover, FFT allows one to extract salient information of feature maps from the frequency domain, which provides additional cues besides visual features. Ref.~\cite{xu2020learning} replaces vanilla downsampling with frequency map of input data, which works as a new form of data pre-processing. In~\cite{NIPS2015_536a76f9}, spectral pooling is proposed, which does pooling in the frequency domain and preserves more information than the regular pooling done in the spatial domain. In~\cite{xiao2020invertible}, the authors propose a novel invertible network to do image rescaling task by embedding lost high-frequency information in the downscaling direction. While most of these methods focus on efficient convolution or processing of data/features maps, our proposed FAT is the first to leverage frequency properties to learn quantization-friendly representations in the weight space, as a bridge to achieve low-bitwidth models.
\section{Frequency-Aware Transformation (FAT)}
\label{sec:method}
\subsection{Preliminary}
\label{sec:preliminary}
% research gap: existing methods focus on spatial info
We introduce two standard quantizers, uniform quantizer and logarithmic quantizer. Traditionally, if the bitwidth $m$ is fixed without learning for all layers, a quantizer is determined by the quantization levels and clip threshold. A full-precision value $x$ is clipped by a float threshold $\alpha$ and then projected onto the range $[0,1]$. Using unsigned quantization as an example, uniform quantization is formulated as:
\begin{equation}
\label{eq:quantization}
\small
{\rm Q}(x) = \triangle \cdot {\rm round}(\frac{{\rm clip}(x, -\alpha, \alpha)}{\triangle}), \quad \triangle=\frac{1}{2^m-1},
\end{equation}
where clipping mitigates the negative effect of extreme values. For uniform quantization, it has quantization levels $\left \{0, \frac{1}{2^m-1},\frac{2}{2^m-1},\cdots,\frac{2^m-2}{2^m-1}, 1 \right \}$. The interval between quantization levels is fixed to be $\frac{1}{2^m-1}$, therefore the same conventional adders and multipliers can be employed during deployment. For logarithmic quantization, it has quantization levels satisfying powers of two $\left \{0, 2^{-2^m+2},2^{-2^m+3},\cdots,2^{-1}, 1 \right \}$. The full-precision value is mapped to its nearest quantization level to get quantized value. Although logarithmic quantizer involves different multipliers for different quantization levels, they can be obtained with cheap \textit{shift} operations. Signed quantization can be extended straightforwardly by deceasing $m$ by 1 bit and symmetrizing the quantization level.

%while bringing non-negligible quantization error since all values larger than $\alpha$ or smaller than $-\alpha$ are quantized to one quantization level ($\alpha $ or $-\alpha$). 

%% Framework: transform (+ quantizer)
\subsection{FAT Framework} 
\label{sec:frame}
Compared with existing approaches that focus on how to design quantizers to fit the full-precision weights, our proposed FAT attempts to generate quantization-friendly representations via a spectral transform. Since the capacity of low-bitwidth model is more restrictive, a good representation should make full utilization of each bit. The transform is supposed to keep the salient information while disregarding unimportant cues. To unify the operation for both convolution and fully-connected layers, we reshape a CNN kernel tensor $\mathcal{W}\in \mathbb{R}^{C_{out}\times C_{in}\times k\times k}$ to a $2$-D matrix denoted as $\mathcal{W} \in \mathbb{R}^{C_{out}\times N}$, where $N = k^2C_{in}$.  Each row denotes a filter. Since different filters are computed separately in convolution, we apply 1-D Discrete Fourier Transform (DFT) $\mathcal{F}(\cdot)$ on each filter to obtain frequency map $\mathcal{W}_f$. The process of DFT  is formalized as:
\begin{equation}
\label{eq:pre_map}
\begin{aligned}
\mathcal{W}_f(i,:) &= \mathcal{F}(\mathcal{W}(i,:)), \\
\mathcal{W}_f(i,k) &= \sum_{n=0}^{N-1}\mathcal{W}(i,n)\cdot e^{-j \frac{2\pi}{N}kn},    
\end{aligned}
\end{equation}
where filter index $i=0,\cdots,C_{out}-1$ and neuron index in filter $k=0,\cdots,N-1$. By this  operation, we encode each convolutional filter with $N$ frequency basis functions. DFT generates complex values instead of real, so we compute the spectral norms of weights at each frequency. The spectral norm $\left \|\mathcal{W}_f   \right \| \in R^{C_{out}\times N}$ reflects the energy on each frequency. As shown in Figure~\ref{fig:ss}, compared with weights in the spatial domain, the energy distribution is much sparser in the frequency domain. Regardless of any filters, the energy in low frequencies are strong, while the energy in high frequencies is weak. This fact is generally observed in different layers and different network architectures used in the Experiments section. It inspires us to learn a soft mask that automatically learns importance from the frequency map, suppressing redundant information flow into low-bitwidth models. Then, a simple quantizer can be applied regardless of layers and architectures.

\begin{figure}
\centering
\includegraphics[scale=0.26]{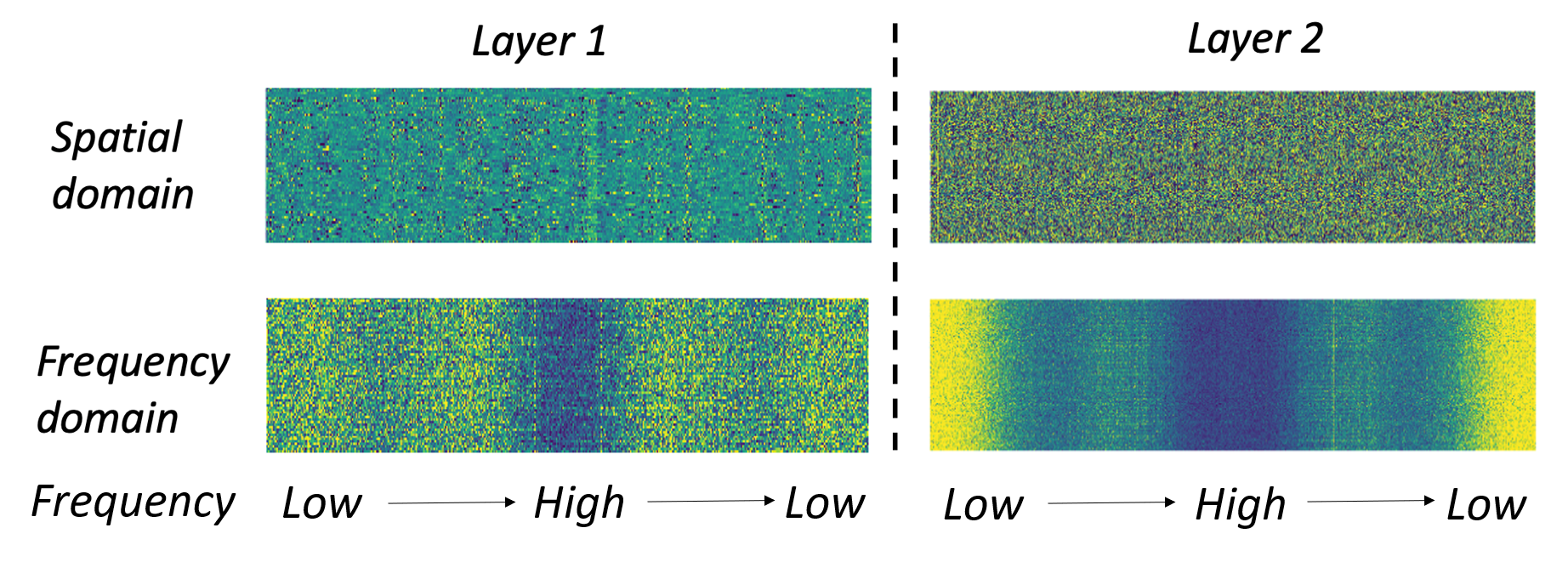}
\caption{Visualization of two flattened weights in spatial domain (top) and frequency domain (bottom). The warmer the color, the higher the value. The density is randomly distributed in the spatial domain, while concentrated on low frequencies in the frequency domain. We then use a soft mask to learn the importance of frequency map of weights in an element-wise manner.}
\label{fig:ss}
\end{figure}

\begin{figure}[t]
\centering
\includegraphics[scale=0.2]{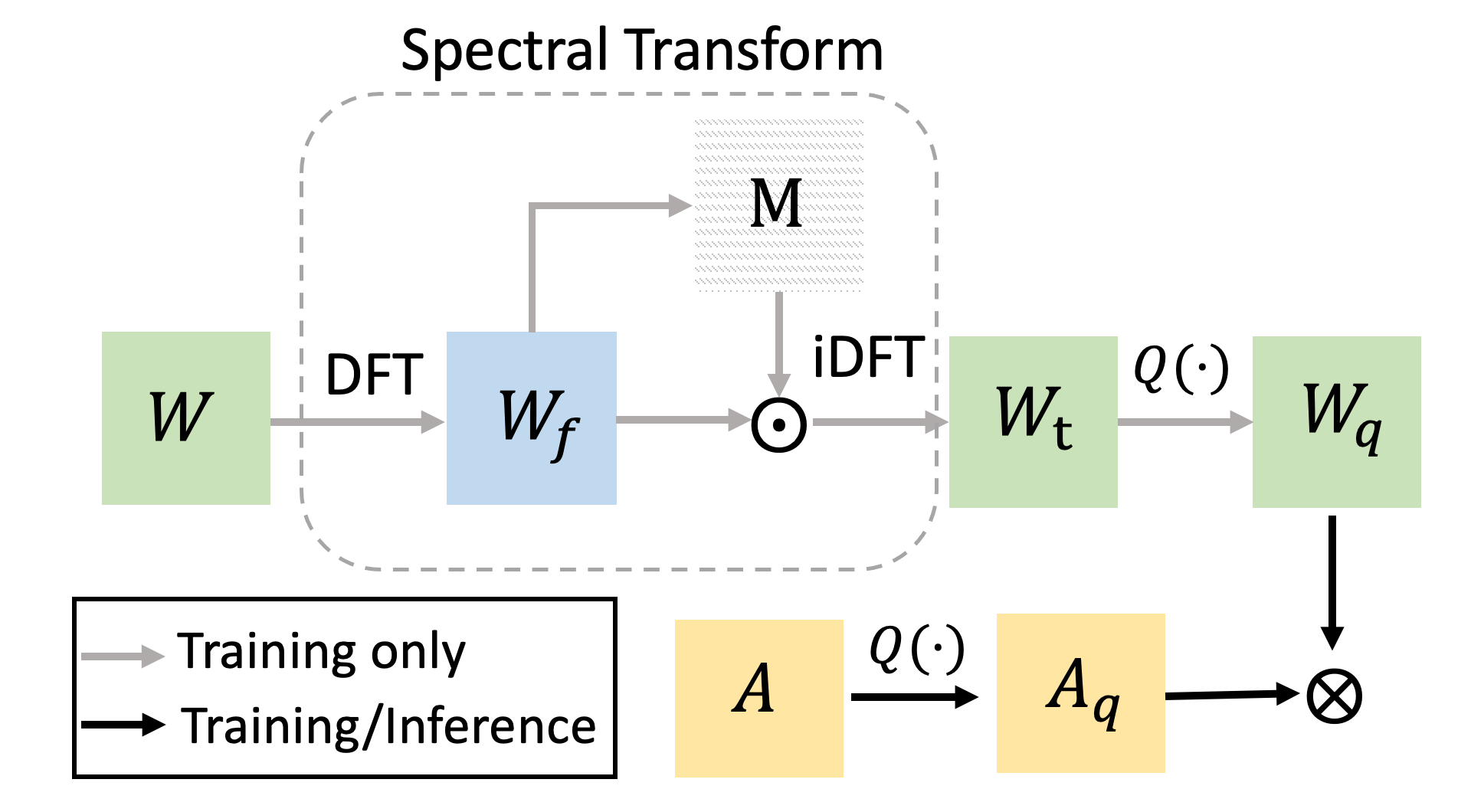}
\caption{Illustration of the proposed quantization process. ``$\mathcal{W}$'' and ``A'' stand for weight and activation. $\rm Q(\cdot)$ is a standard quantizer. We use $\odot$  and $\otimes$ to denote Hadamard product and convolution operation, respectively. }
\label{fig:process}
\end{figure}

The complete workflow is visualized in Figure~\ref{fig:process}. We use a trainable mask $M \in \mathbb{R}^{C_{out} \times N}$ to find and distinguish the quantization-friendly and quantization-useless components in the weights, and the useless components will be softly deactivated with coefficient ranging from $0$ to $1$. This step can be formalized by the following equations:
\begin{equation}
\label{eq:3}
M = {\rm Sigmoid} (W_m^T (\left \|\mathcal{W}_f   \right \|),\\
\end{equation}
\begin{equation}
\label{eq:4}
\hat{\mathcal{W}_f} = M \odot \mathcal{W}_f,
\end{equation}
where $W_m \in \mathbb{R}^{C_{out} \times {C_{out}}}$ linearly maps the power spectral map. Sigmoid function ensures all coefficients learned in the mask are in $[0,1]$. And $\odot$ represents the element-wise multiplication. The learning process of the mask can be viewed as a  self-attention mechanism in the frequency domain. The mask retains useful frequency components and ignores trivial components. Finally, we map the weights back to the spatial domain: 
\begin{equation}
\mathcal{W}_{t}(i,:) = \mathcal{F}^{-1}(\hat{\mathcal{W}_f}(i,:)),
\label{eq:5}
\end{equation}
where $\mathcal{F}^{-1}$ is the inverse Discrete Fourier Transform (iDFT). The transformed weights will be quantized by a standard quantizer. Then, convolutional weights will be reshaped to 4-D for convolutional operation.
\begin{figure}[t]
\centering
\includegraphics[scale=0.35]{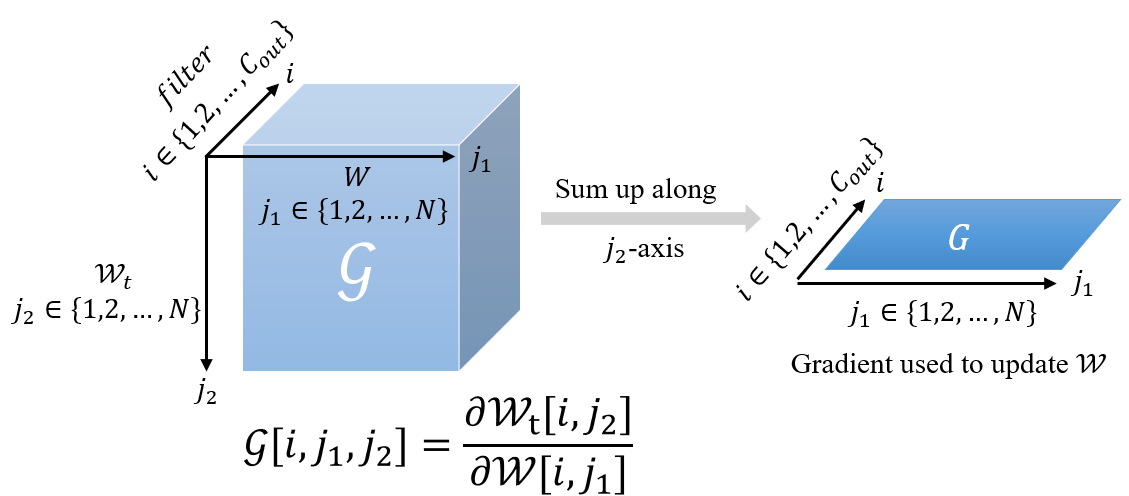}
\caption{Illustration of the process of updating gradient for flattened convolution weights during transform.}
\label{fig:gradient3d}
\end{figure}

% algorithm: FAT 
\begin{algorithm}[t]
\begin{algorithmic}
\caption{The forward and backward processes of FAT applied on one convolutional layer.}
\label{alg1}
\STATE {\bfseries Input:} weight $\mathcal{W}$ and activation $A$, bitwidth $m$;
\STATE {\bfseries Output:} quantized  output $O$
\STATE {\bfseries Parameters:} soft mask $M$, threshold $\alpha_{\mathcal{W}}$ , $\alpha_{A}$ 
\STATE \textbf{Feed Forward:}
\STATE $\mathcal{W}_f = \mathcal{F}(\mathcal{W})$, map the full-precision weight from spatial domain to frequency domain;
\STATE $\mathcal{W}_{t} = \mathcal{F}^{-1}(\rm M \odot \mathcal{F}(\mathcal{W}))$, where $\rm M$ is generated by Eq.~\ref{eq:3} to adjust the passing proportion in different frequency bases;
\STATE $\mathcal{W}_q= {\rm Q}({{\rm clip}}(\mathcal{W}_{t}, -\alpha_{\mathcal{W}},  \alpha_{\mathcal{W}}))$, where ${\rm Q}(\cdot)$ is a standard quantizer applied on the transformed weight $\mathcal{W}_{t}$;
\STATE $A_q = {\rm Q}({{\rm clip}}(A, -\alpha_{A},  \alpha_{A}))$;
\STATE $O = {\rm Conv}(\mathcal{W}_q, A_q)$;
\STATE \textbf{Backward Propagation:}
\STATE $\frac{\partial \mathcal{L}}{\partial \mathcal{W}_{t}} $ = $\frac{\partial \mathcal{L}}{\partial \mathcal{W}_{q}}$ $\cdot$ $\frac{\partial \mathcal{W}_{q}}{\partial \mathcal{W}_{t}}$ $\approx$ $\frac{\partial \mathcal{L}}{\partial \mathcal{W}_{q}} \cdot \mathbb{I} [|\mathcal{W}_{t}| < \alpha_{\mathcal{W}}]$;
\STATE $\frac{\partial \mathcal{L}}{\partial M} $ = $\frac{\partial \mathcal{L}}{\partial \mathcal{W}_{t}}$ $\cdot$ $\frac{\partial \mathcal{W}_{t}}{\partial M}$, where the soft mask learns different frequency clues jointly to update;
\STATE	$\frac{\partial \mathcal{L}}{\partial \mathcal{W}} $ = $\frac{\partial \mathcal{L}}{\partial \mathcal{W}_{t}} $  $\cdot$ $\frac{\partial \mathcal{W}_{t}}{\partial \mathcal{W}}$, where the discretization function learns different frequency clues jointly to update;
\STATE  $\frac{\partial \mathcal{L}}{\partial  \alpha_{\mathcal{W}}} $ = $\sum{ \frac{\partial \mathcal{L}}{\partial \mathcal{W}_q} \cdot {\rm sign}(\mathcal{W}_{t}) \cdot \mathbb{I} [|\mathcal{W}_{t}| > \alpha_{\mathcal{W}}]}$;
% 	\STATE  $\frac{\partial \mathcal{L}}{\partial a} $ = $\frac{\partial \mathcal{L}}{\partial a_q}$ \  $I [|a| < \alpha_{a}]$;
\STATE  $\frac{\partial \mathcal{L}}{\partial  \alpha_{A}} $ = $\sum{\frac{\partial \mathcal{L}}{\partial A_q} \cdot \mathbb{I} [A > \alpha_{A}]}$; \\
// During inference, the transform is removed, the quantized model only uses an uniform/logarithmic quantizer.
\end{algorithmic}
\end{algorithm}

\subsection{Analyses of the Proposed Framework}
We provide theoretical insights showing FAT enables smaller quantization errors and more informative backpropagation gradient via a spectral transform that incorporates structural properties of various frequencies. We also explore the relationship of our approach with various representative schemes. The detailed proof of theorems and derivations of all gradients involved in FAT are detailed in the Appendix.%
\begin{figure*}[t]
\centering
\subfigure[STE]{
\begin{minipage}[t]{0.32\linewidth}
\centering
\includegraphics[width=2.3in, height=2.5in]{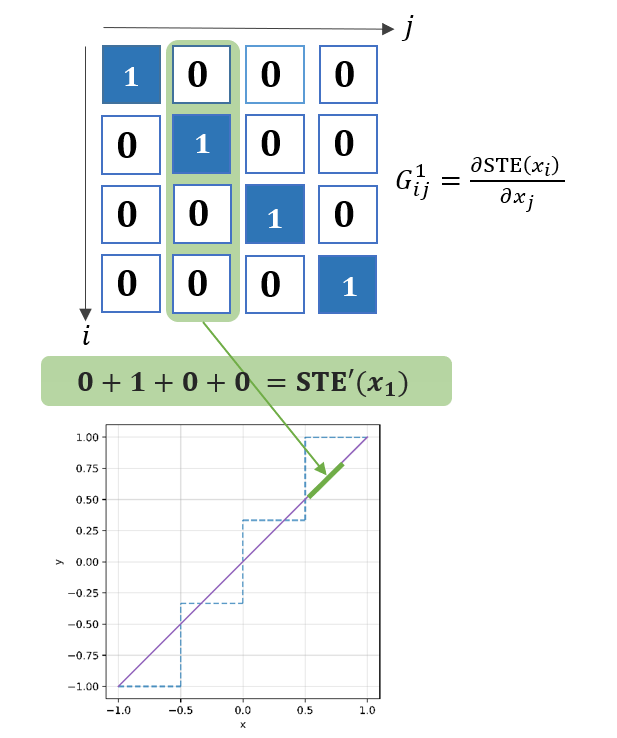}
%\caption{fig1}
\end{minipage}%
}%
\subfigure[DSQ]{
\begin{minipage}[t]{0.32\linewidth}
\centering
\includegraphics[width=2.3in, height=2.5in]{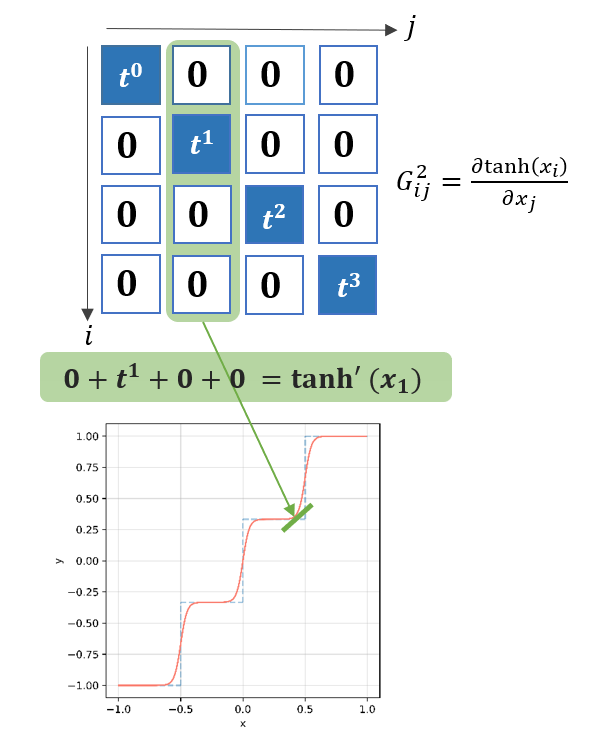}
%\caption{fig2}
\end{minipage}%
}%
\subfigure[FAT (Ours)]{
\begin{minipage}[t]{0.32\linewidth}
\centering
\includegraphics[width=2.3in, height=2.5in]{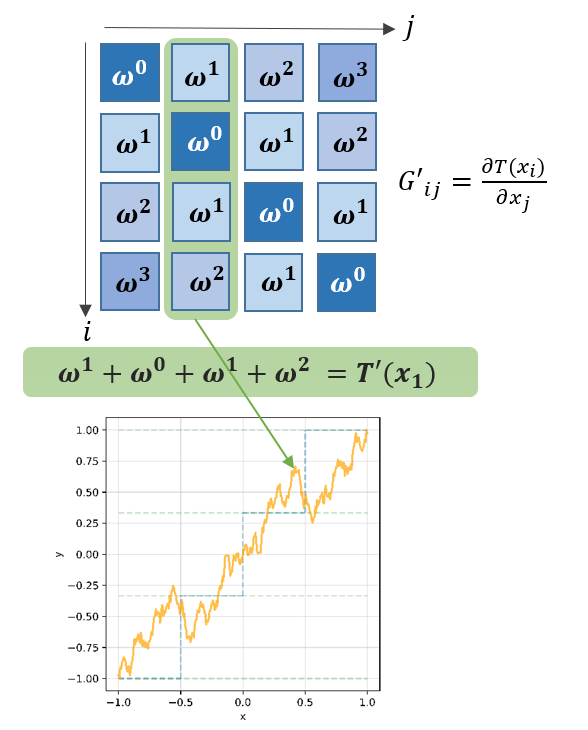}
%\caption{fig2}
\end{minipage}
}%
\caption{\textbf{Top:} Comparison of discretization gradient matrix $\frac{\partial \mathcal{W}_q}{\partial \mathcal{W}}$ in STE, DSQ \cite{gong2019differentiable} and ours, given any filter (take the number of neurons $N=4$ as example). \textbf{Bottom:} The designed quantization approximation functions. The derivatives of these functions are the corresponding discretization gradients. For each neuron, the  derivative is equal to the column sum of the top gradient matrix. Instead of quantizing each neuron independently in STE or DSQ (diagonal gradient matrix), our method considers all neurons in quantization and utilizes information of all frequencies.}
\label{fig:gradient_comp}
\end{figure*}

\subsubsection{Quantization Error Reduction}
%%%%%%%%%%%% begin: Rui's version
Quantization error can directly reflect the quality of quantization. Following~\cite{banner2019post}, we employ the expected mean-squared-error (MSE) between the full-precision weights and its quantized version as the metric to evaluate the quantization error. In this section, we theoretically show that the transformed weight $\mathcal{W}_t$ is guaranteed to have a smaller quantization error than the original weight $\mathcal{W}$.  Without loss of generality, weight satisfies zero-mean distribution. As proven in Theorem~\ref{theorem:mse},
\begin{theorem}
\label{theorem:mse}
Assume weight $\mathcal{W}$ satisfies ${\rm{Laplace}(0,b)}$ distribution. Then the following two inequalities hold for both uniform and logarithmic quantization:
\begin{subequations}
\begin{equation}
\label{eq:inequal_1_sm}
\max(|\mathcal{W}_t|) < \max(|\mathcal{W}|),
\end{equation}
\begin{equation}
\label{eq:inequal_2}
E[(\mathcal{W}_t-\rm Q(\mathcal{W}_t))^2] \leq E[(\mathcal{W}-\rm Q(\mathcal{W}))^2].
\end{equation}
\end{subequations}
\end{theorem}
We show that the expected MSE between any given full-precision weights and its quantized version can be generally approximated as the function of the clipping threshold $\alpha$ and the amplitude $a=\max(|\mathcal{W}_t|)$. For uniform quantization and logarithmic quantization, the quantization error can be analytically written as: 
\begin{equation}
\label{eq:mse_u}
\begin{aligned}
E[(\mathcal{W}_t-{\rm Q}_u(\mathcal{W}_t))^2] &= f_{u}(\alpha,a) = f_{clip} + f_{q\_u} \\
f_{q\_u} &=\frac{\alpha^2}{3\cdot 2^{2m}}
\end{aligned}
\end{equation}

\begin{equation}
\label{eq:mse_log}
\begin{aligned}
E[(\mathcal{W}_t-{\rm Q}_{\log}(\mathcal{W}_t))^2] &= f_{log}(\alpha,a) = f_{clip} + f_{q\_log} \\
f_{q\_log} &= \frac{\alpha^2}{84}\cdot(1+3\cdot 2^{-3\cdot 2^{m-1}+1})
\end{aligned}
\end{equation}

\begin{equation}
\begin{aligned}
f_{clip} &= e^{-\frac{a}{b}}\cdot [2\alpha a - (b+a)^2 - (b-\alpha)^2] b^2 \cdot e^{-\frac{\alpha}{b}},
\end{aligned}
\end{equation}

where $m$ is the bitwidth, $\rm Q_u(\cdot)$ and $\rm Q_{\log}(\cdot)$ represent uniform and logarithmic quantization, respectively.

From  Eq.~\ref{eq:mse_u} and \ref{eq:mse_log}, the quantization errors are formed by two terms, namely, quantization noise and clipping noise. The clipping noise term is approximately the same for both $\rm Q_u(\cdot)$ and $\rm Q_{log}(\cdot)$. In the proposed FAT, the mask helps tighten the weights $\mathcal{W}$ towards zero. Alternatively speaking, applying the mask on $\mathcal{W}$ leads to the result that $\max(|\mathcal{W}_t|)) \leq \max(|\mathcal{W}|)$. Since the data range of weights shrinks after transform, the quantization resolution increases for all transformed weights. On the other hand, the informative components of full-precision weights are kept by passing important frequencies in the mask. The detailed proof is available in the Appendix. We also show that simply adjusting the standard deviation to tighten the weights does not work, since important components are not guaranteed to pass to the quantized model this way.

%%%%%%%%%%%% end: Rui's version

% Quantization error directly reflects the quality of quantization, here we use the expected mean-square-error (MSE) between the full-precision weights $X$ and its quantized version $Q(X)$ to evaluate quantization error. By assuming the values are rounded to the midpoint of the region \cite{banner2019post}, quantization error can be formalized as:
% \begin{equation}
% \small{
% \begin{aligned}
% & E[(w-Q(w))^2] =  \int_{-\infty }^{-\alpha}f(w)(w+\alpha)^2 + \\ & \sum_0^{2^m-1}\int_{-\alpha + i\Delta }^{-\alpha + (i+1)\Delta }f(w)(w-q_i)^2dx +\int_{\alpha}^{\infty}f(w)(w-\alpha)^2,
% \end{aligned}
% }
% \label{eq:6}
% \end{equation}
% where $f(\cdot)$ is probability density function. The first and third terms evaluate error introduced outside the clipping threshold, and the second term evaluate the error within the threshold. Note that weights always satisfies zero-center bell distribution, we use Laplace(0,b) distribution to approximate the  probability density function. When using uniform quantization, $E[(w-Q(w))^2]\approx 2b^2e^{-\frac{\alpha}{b}} + \frac{\alpha^{2}}{3\cdot2^{2m}}$. By plugging the proposed transform $T(\cdot)$, the quantization error is consistently reduced, \textit{i.e.} $E[(T(w)-Q(T(w)))^2] < E[(w-Q(w))^2]$. The quantization error becomes: \textbf{xx}. The detailed derivation and simulation is written in the Appendix \textbf{X}. {\textcolor{red}{modify it here in the format of theorem}}

\subsubsection{Informative Discretization gradient}
%   $\mathcal{W} \in \mathbb{R}^{C_{out}\times N}$
Standard quantization is not differentiable and therefore does not readily allow backpropagation. In order to learn a good quantizer, we should carefully design the discretization gradient $\frac{\partial \mathcal{W}_q}{\partial \mathcal{W}}$. The gradient should not only be computable, but also retain information during quantization. Before analysing our discretization gradient and make comparisons with former approaches, we show how to compute the gradient of transform. As shown in Figure~\ref{fig:gradient3d}, given $i$-th vectorized filter $\mathcal{W}(i, :) \in \mathbb{R}^{N}$ and its transformed version $\mathcal{W}_{t}(i, :) \in \mathbb{R}^{N}$, the gradient $\frac{\partial \mathcal{W}_t(i, :)}{\partial \mathcal{W}(i, :)}$ for filter $i$ is a matrix $\in \mathbb{R}^{N\times N}$. Given $j_1$-th neuron in the $i$-th filter, the column in gradient tensor $\mathcal{G}[i,j_1,:]=[\frac{\partial \mathcal{W}_t(i, 1)}{\partial \mathcal{W}(i, j_1)}, \frac{\partial \mathcal{W}_t(i, 2)}{\partial \mathcal{W}(i, j_1)}, \cdots, \frac{\partial \mathcal{W}_t(i, N)}{\partial \mathcal{W}(i, j_1)}]$ reflects the effects of  all the transformed neurons on the original neuron $j_1$. Therefore, this whole gradient tensor $\mathcal{G}$ is summed along the $j_2$-axis (column) to obtain the same shape as the original weight, which is used for weight updating during backpropagation.

 We visualize the gradient matrix and corresponding approximation function in Figure~\ref{fig:gradient_comp}. STE~\cite{bengio2013estimating} is a rough approximation for discretization. It defines $\frac{\partial \mathcal{W}_q}{\partial \mathcal{W}} = \mathbb{I} [-1 < \frac{\mathcal{W}}{\alpha} < 1]$. It assigns the gradients of all the values within clip threshold as 1, and outliers as 0. The antiderivative function $f(\cdot)$ of the aforementioned gradient is $f(x) = {\rm clip}(x, -1, 1)$. DSQ \cite{gong2019differentiable} approximates the gradient by a set of hyperbolic tangent functions $f(x) = \frac{1}{Z} \tanh (kx)$, where $Z$, $k$ are normalization factor and handcrafted coefficient, respectively. The discretization gradient is the derivative of these hyperbolic tangent functions. Note the functions used in DSQ is different from ours from two sides: 1) DSQ just modifies the gradient during backward pass, but does not affect quantization results during forward pass, 2) DSQ treats each neuron separately. 
 
 With the bridge of our proposed transform, the discretization gradient becomes $\frac{\partial \mathcal{W}_q}{\partial \mathcal{W}} = \frac{\partial \mathcal{W}_q}{\partial \mathcal{W}_t} \frac{\partial \mathcal{W}_t}{\partial \mathcal{W}} $. The first term $\frac{\partial \mathcal{W}_q}{\partial \mathcal{W}_t}$ is approximated with STE, and we can utilize the transform $\rm T(\cdot)$ to adjust the second term during training. Given a convolutional filter with index $i$, the gradient of the $k_1$ transformed neuron to $k_2$ original neuron $\frac{\partial \mathcal{W}_{t}(i,k_1)}{\partial \mathcal{W}(i,k_2)}$ during backpropagation considers clues from all frequency bases,
\begin{small}
\begin{equation}
\label{eq:grad_wde_w}
\frac{\partial \mathcal{W}_{t}(i,k_1)}{\partial \mathcal{W}(i,k_2)} = \frac{1}{N}\sum_{n=0}^{N-1} M(i,n) \cdot \cos{\left( \frac{2\pi (k_1 - k_2) n}{N}\right)},%
\end{equation}%
\end{small}%
where $i=0,\cdots,C_{out}-1$ and $k_1,k_2=0,\cdots,N-1$. The soft mask $M$ in our transform is used to learn the importance of weights on different frequency bases. By denoting $\cos{\left( \frac{2\pi i n}{N}\right)} = \omega^{i}$, we can see that our gradient matrix is a non-diagonal matrix, which shows that our model considers cross-neuron dependencies during training.

If the mask $M$ is an all-ones matrix, all frequencies are allowed to pass in 100\%, then the transform $\rm T(\cdot) = \mathcal{F}^{-1}\circ  M  \odot\mathcal{F} (\cdot)$ becomes an identity map, $\mathcal{W} \rightarrow \mathcal{W}$. In this case, the gradient matrix in Eq.\ref{eq:grad_wde_w} degenerates to an identity matrix and the discretization gradient degenerates to the STE.

% % grad: theorem 1 
% \begin{theorem}
% \label{theorem:w_de_M}
% The gradient of $w_{t}$ to $M$ during backward propagation is $\frac{\partial w_{t}}{\partial M} = \sum_{k_1=0}^{N-1} \sum_{i,k_2}\frac{\partial w_{t}(i,k_1)}{\partial M(i,k_2)}$, with 
% % eq: w_de/M element-wise
% \begin{small}
% \begin{equation}
% \label{eq:grad_wde_M}
% \frac{\partial w_{t}(i,k_1)}{\partial M(i,k_2)} = \frac{1}{N} \sum_{n=0}^{N-1} w(i,n) \cdot \cos{\left(\frac{2\pi (k_1 - n) k_2}{N}\right)},
% \end{equation}
% \end{small}
% where $i=0,\cdots,C_{out}-1$ and $k_1,k_2=0,\cdots,N-1.$
% \end{theorem}

% % fig: calculate grad
% \begin{figure}
% \centering
% \includegraphics[scale=0.51]{files/grad_matrix_method.png}
% \caption{The calculation process of $\frac{\partial w_{t}(i,:)}{\partial w(i,:)}$ when $N=4$. It is noticeable that the gradient matrix is symmetric.}
% \label{fig:grad_calculation}
% \end{figure}

\begin{table*}[th]
\scriptsize
\begin{center}
\setlength{\tabcolsep}{12mm} 
\caption{Comparison on ImageNet dataset with different bitwidths for weight (W) and activation (A). We use ``*'' and ``m'' to mark our implementation results and mixed-precision, respectively. Acc@1 and  Acc@5 denote top-1 and top-5 accuracy in percentage.}
\label{table1}
\begin{tabular}{l|l|lll}
\toprule
Architecture & Methods & W/A  & Acc@1 & Acc@5 \\
\midrule
\multirow{12}{*}{ResNet-18}&Full-precision & 32/32 & 69.6 & 89.0\\
% ABC-Net &5/5  & 65.0 &85.9\\
~&RQ &8/8&  70.0 & 89.4\\
~&LSQ &8/8& 71.1 & 90.1 \\
~&UNIQ &4/8& 67.0& - \\
% DoRefa-Net &5/5 &68.4 &88.3\\
~&DJPQ & 4/8 (m) & 69.3 & - \\
~&PACT &5/5 &69.8 &89.3\\
~&LQ-Net & 4/4 & 69.3 & 88.8 \\
~&DSQ  & 4/4 & 69.5 & -\\
~&APoT$^*$ &4/4 &69.9 &89.3\\
\cdashline{2-5}
% 71.0 89.8
~&FAT(Ours) &5/5 & \textbf{70.8} & \textbf{89.7 }\\
~&FAT(Ours) &4/4 &\textbf{70.5 }& \textbf{89.5} \\
~&FAT(Ours) &3/3 &\textbf{69.0} & \textbf{88.6} \\
% FAT(Ours) &2/2 & - & - \\
\hline
\multirow{9}{*}{ResNet-34}&Full-precision & 32/32 & 73.7 & 91.3 \\ 
% ABC-Net & 5/5 & 68.4  & 88.2 \\
~& LSQ & 8/8 & 74.1 & 91.1 \\
~& BCGD & 4/4 &73.4& 91.4 \\
~& QIL & 4/4 &73.7& - \\
~&DSQ  & 4/4 & 72.8 & -\\
~&APoT$^*$ & 4/4  &73.0 & 91.0 \\
\cdashline{2-5}
~&FAT(Ours) &5/5 & \textbf{74.6} & \textbf{91.8 }\\
~&FAT(Ours) & 4/4  & \textbf{74.1}  & \textbf{91.8} \\
~&FAT(Ours) &3/3 & \textbf{73.2} & \textbf{91.2} \\
%FAT(Ours) &2/2 & - & - \\
\hline
\multirow{10}{*}{MobileNet-V2}&Full-precision & 32/32 &  71.7 & 90.4 \\
% DFQ & 8/8 & 71.2(70.43?) & - \\
% DBQ & 8/8 & 70.54& - \\
~&HAQ &4/32 (m) & 71.4& 90.2 \\
~&HMQ  &4/32 (m) & 70.9 & - \\
~&DQ & 4/8 (m) & 68.8 & - \\
~&DJPQ & 4/8 (m) & 69.0 & - \\
~&PACT &4/4  &61.4&  -\\
~&DSQ  &4/4 & 64.8 & -\\
% LSQ & 4/4 & 70.6 & - \\
% APoT$^*$ & 4/4 & 69.5 & 88.7\\
\cdashline{2-5}
~&FAT(Ours) &5/5 & \textbf{69.6} & \textbf{89.2} \\
~&FAT(Ours) &4/4 &\textbf{69.2}& \textbf{88.9} \\
~&FAT(Ours) &3/3 &\textbf{62.8} & \textbf{84.9} \\
% FAT(Ours) &2/2 & - & - \\
\bottomrule
\end{tabular}
\end{center}
\end{table*}

%% Complexity 
\subsection{Complexity Analysis}
\label{sec:complexity}
The proposed FAT is summarized in Algorithm \ref{alg1}. Given a convolution layer with a $4$-D weight tensor $\mathcal{W} \in \mathbb{R}^{C_{out} \times C_{in} \times k \times k}$, and an input of size $(C_{in}, H, W)$, where $H$ and $W$ are height and width of the input, respectively. The number of multiply-accumulate (MAC) operations in the common convolution is $N_{mac} = H\cdot W \cdot C_{out} \cdot C_{in} \cdot k^2$. Denoting $N=C_{in} \cdot k^2$, the complexity of $\mathcal{F}(\mathcal{W}(i,:))$ and $\mathcal{F}^{-1}(\hat{\mathcal{W}}_f(i,:))$ are both $C_{out}\cdot N\cdot \log (N)$. Regardless of the data batch size, the number of MACs in magnitude computation, element-wise product and fully-connected mapping sum up to $4\cdot C_{out} \cdot N + C_{out}^2\cdot N$. The extra MACs introduced by the transform is $\Delta N_{mac} = 2\cdot C_{out}\cdot N\cdot \log (N) + 4\cdot C_{out} \cdot N + C_{out}^2\cdot N$. For example, assuming $\mathcal{W} \in \mathbb{R}^{256 \times 3 \times 3 \times 3}$ and the input size is $(3, 224, 224)$, then $\frac{\Delta N_{mac}}{N_{mac}} = \frac{2\cdot\log(3\cdot3\cdot3)+4+256}{224\cdot 224} \approx 0.0054$. Hence, the transform introduces negligible training cost.

% \textit{\begin{algorithm}[th]
% 	\caption{Our quantization framework in the Post-Training manner for weights. [need experiments]}
% 	\label{alg2}
% 	\begin{algorithmic}
% 		\STATE {\bfseries Input:} signed weight $w$, bitwidth $m$, a few iterations $T$, MSE $\mathcal{E} $;
% 		\STATE {\bfseries Output:} mask M, adjusted amplitude  $\mathcal{S}$ and threshold $\alpha_{w}$ 
% 		\STATE $\mathcal{S}  = max(|w|)$;
% 		\STATE  Initialize $\alpha_{w} = \mathcal{S}$, $\mathcal{E} = inf$;
% 		\STATE Estimate the distribution parameter $b$ or $\sigma$;
% 		\WHILE{t $<$ T}
% 		\STATE $\mathcal{S}^* = f(\mathcal{S}, \alpha_{w})$ using $f(\cdot)$ in Eq.x;
% 		\IF {$ MSE(\mathcal{S}^*, \alpha_{w}) < \mathcal{E}$}
% 			\STATE $\mathcal{E} =  MSE(\mathcal{S}^*, \alpha_{w})$;
% 			\STATE $\mathcal{S} =\mathcal{S}^*$;
% 		\ENDIF
% 		\STATE $\alpha_{w}^* = g(\mathcal{S}, \alpha_{w})$ using $g(\cdot)$ in Eq.x; 
% 		\IF {$ MSE(\mathcal{S}, \alpha_{w}^* ) < \mathcal{E}$}
% 						\STATE $\mathcal{E} =  MSE(\mathcal{S}, \alpha_{w}^*)$;
% 						\STATE$\alpha_{w} = \alpha_{w}^*$;
% 		\ENDIF
% 		\ENDWHILE
% 		\\
% 		// After weight quantization using  $\mathcal{S}$ and  $\alpha_{w}$, we use some representative data to calibrate  activation threshold $\alpha_{a}$.

% 	\end{algorithmic}
% \end{algorithm}}

\begin{table}[ht]
\scriptsize
    \begin{center}
        \setlength{\tabcolsep}{3.3mm} 
		\caption{Comparison on CIFAR-10 dataset with different bitwidths for weight (W) and activation (A).}
        \label{table2}
        \small{
		\begin{tabular}{l|l|ll}
			\toprule
			Architecture & Methods & W/A  & Accuracy \\
			\midrule
			\multirow{7}{*}{VGG-Small} & Full-precision &  32/32& 93.1 \\
            % TWN  & 2/32 & 92.56  \\
            ~&RQ  & 8/8   & 93.3 \\
            ~&DJPQ  & 4/8 (m) &  91.5 \\
            ~&RQ  &4/4    & 92.0 \\
            ~&WAGE  & 2/8  & 93.2	\\	
% 			FAT(Ours) & 5/5  & 94.12 (94.19)\\
            \cdashline{2-4}
            ~&FAT(Ours) & 4/4  &\textbf{94.4}\\
            ~&FAT(Ours)  & 3/3  & \textbf{94.3}	\\
            % FAT(Ours)  & 2/2  &\textbf{ 93.3 (93.3)}	\\
            \hline
			\multirow{6}{*}{ResNet-20} & 
			Full-precision & 32/32 & 91.6 \\ 
            ~&DSQ  & 1/32  & 90.2 \\
            % ~&LQ-Net &  3/3 & 91.6 \\
            % LQ-Net &2/2 & 90.2\\
            ~&PACT & 4/4 & 90.5 \\
            ~&APoT & 4/4  & 92.3 \\
            % FAT(Ours) & 5/5  & 93.03 (91.89)\\
            \cdashline{2-4}
			~&FAT(Ours)  & 4/4  & \textbf{93.2}	\\	
			~&FAT(Ours)  & 3/3  & \textbf{92.8}	\\	
% 			FAT(Ours)  & 2/2  & \textbf{89.6 (91.1)}	\\
	        \hline
			\multirow{5}{*}{ResNet-56} &
			Full-precision & 32/32 &  93.2 \\
			~&PACT & 2/32 & 92.9 \\
			~&APoT & 4/4 & 94.0 \\
% 			FAT(Ours) & 5/5 & 94.12 (94.13)\\
            \cdashline{2-4}
			~&FAT(Ours) & 4/4 & \textbf{94.6}\\
			~&FAT(Ours)  & 3/3  & \textbf{94.3}	\\
% 			FAT(Ours)  & 2/2  & \textbf{92.3 (93.3)}	\\
            \bottomrule
		\end{tabular}
		}   
	\end{center}
	\vspace{-3mm}
\end{table}

%% Experiments
\section{Experiments}
\label{sec:exp}
We evaluate the effectiveness of the proposed FAT on two commonly used datasets, CIFAR-10 \cite{krizhevsky2009learning} and ImageNet-ILSVRC2012 \cite{russakovsky2015imagenet}. CIFAR-10 is an image classification dataset with 10 classes. ImageNet is a large dataset with 1.3M training images and 50k validation images. We adopt standard training-test data split for both datasets. 
% We also provide extensive ablation study about what the transform learns for quantization-friendly representation from the frequency domain. 

VGG-small \cite{zhang2018lq}, ResNet-20 and ResNet-56 \cite{he2016deep} are used on  CIFAR-10 dataset. ResNet-18, ResNet-34 and MobileNetV2 \cite{sandler2018mobilenetv2} are used on  ImageNet dataset. 

The proposed FAT is built on Pytorch framework.  We compare FAT with state-of-the-art approaches, including WAGE \cite{wu2018training}, LQ-Net \cite{zhang2018lq}, PACT\cite{choi2018pact}, RQ \cite{louizos2018relaxed}, UNIQ \cite{baskin2018uniq}, DQ \cite{tung2018deep}, BCGD \cite{baskin2018nice} \cite{tung2018deep}, DSQ \cite{gong2019differentiable}, QIL \cite{jung2019learning}, HAQ \cite{wang2019haq},  APoT \cite{li2020additive}, HMQ \cite{habi2020hmq} DJPQ \cite{wang2020differentiable}, LSQ \cite{esser2020learned}.

We evaluate the model by trade-off among accuracy, model size and bit-operation (BOP). BOP is a general metric that considers both the bitwidth and the number of multiply-accumulate (MAC) operations \cite{wang2020differentiable}. The formula for bit-operation is ${\rm BOP}=m_w m_a  {\rm MAC}$, where $m_w$ and $m_a$ are bitwidths of weight and activation, respectively. A smaller BOP means lighter computation. We report the proposed transform on a uniform quantizer. The performance on logarithmic quantizer, a brief categorization of state-of-the-art methods and  training details are elaborated in the Appendix.

\subsection{Experimental Results}
As shown in Tables~\ref{table1} \&~\ref{table2}, FAT surpasses previous state-of-the-art methods in accuracy without using a complicated quantizer or training tricks, such as learnable quantization stepsize in LSQ, reinforcement learning-based quantization policy in HAQ and arbitrary-bit precision in HMQ and DJPQ, etc. These methods attempt to learn powerful non-uniform quantizers to fit the distribution of original full-precision data, thereby decreasing quantization error. Compared with the approaches above, the proposed FAT achieves state-of-the-art performance without bells and whistles, trading off accuracy and bitwidth. Instead of using high bitwidth like 8-bit, we show that our method enables the commonly used networks like ResNet, VGG-Small and MobileNet to have acceptable performance in 3-bit or 4-bit setting.

Tables~\ref{table1} \&~\ref{table2}  show  the power of using representation transform before quantization. By viewing the full-precision weights as images and then mapping them to another space where unimportant frequency bases are deactivated, we are capable of using a simple uniform/logarithmic quantizer to achieve competitive performance. It indicates the importance of bridging the full-precision weight to a quantization-friendly representation before quantization, especially for low bitwidth setting like 3-bit integer quantization. Our results even outperform full-precision model, since quantization has regularization effect during the training process.

% Note use of \abovespace and \belowspace to get reasonable spacing
\subsection{Ablation Study}
\subsubsection{Hardware Performance}
Since the weights after quantization-aware training are fixed, the transform is removed during inference. During inference, the  FAT is as light as using a uniform/logarithmic quantizer on activation. FAT does not need to store extra parameters in quantizers, and employs unified quantization scheme for all layers. Hence, the proposed FAT reduces the computation power compared with most previous methods. As shown in Table~\ref{table3}, we compare the model size and bit-operation among different quantization methods. When quantizing both weight and activation to 4 bit, our method achieve 7.7$\times$, 7.9$\times$, 6.7$\times$ model size compression and 54.9$\times$, 58.5$\times$, 45.7$\times$ bit-operation reduction against full-precision ResNet-18, ResNet-34 and MobileNetV2, respectively. 

In addition, employing a standard quantizer for all layers is hardware-friendly. For instance, if using adaptive quantization levels or mixed-precision, different quantizers need to be adopted per layer and/or channel, which greatly increases the difficulty of hardware deployment on, e.g., ARM CPU, FPGAs, etc.

\subsubsection{Suppressed Frequencies in Quantization}
We wonder what frequencies the quantized model prefers to keep or discard. In Figure~\ref{fig:mask}, we visualize the frequency maps of weights and corresponding learned masks in 4 layers. The visualization demonstrates that the high frequencies not only have weak spectral density in the frequency map, but are also suppressed from full-precision model to quantized model. In traditional image denoising, noises usually have weak spectral density compared with visual feature. By removing the frequencies with weak density, we could reduce the noise in the image. Here we learn the frequencies' importance in the weight space, whereby the transform has similar effect as denoising, i.e., curtail redundant weights flowing to a capacity-limited low-bitwidth model. It also indicates that low frequency components in neurons are important for quantization.

\setlength{\tabcolsep}{2mm} 
\begin{table}[t]
\scriptsize
	\begin{center}
		\caption{Hardware performance in terms of model size and bit-operation on ImageNet dataset, where ``C.R.'' denotes corresponding compression rate. }
		\label{table3}
        \begin{tabular}{l|lllll}
			\toprule
			Methods & bitwidth  &  Size (MB) & C.R. & BOPs (G)  & C.R. \\
			\midrule
			\textbf{ResNet-18} \\
			F.P. & 32/32 & 46.8 & 1x & 1863.7& 1x \\
% 			apot  &4/4 &- & - & -\\
% 			RQ  & 8/8 &11.8 & 4.0x & 115.8 & 16.1x \\
			%LSQ &8/8 & 11.8 & 4.0x &971.4 & 1.9x\\
		  %  LSQ   &3/3&  - &- & 136.63 \\
		  %  UNIQ  &4/8 & 6.3 & 7.4x         & 58.2& 32.0x   \\
		    APoT &4/4& 6.3 & 7.4x     & 36.5 & 51.0x\\
		    DSQ &4/4 & 6.3 & 7.4x     & 36.5 &  51.0 x\\
		  %  DJPQ & 4/8 (m) & 6.1 & 7.4x & 43.6& 42.7x \\
			FAT(Ours) &4/4 &\textbf{6.1} & \textbf{7.7x} & \textbf{33.9}& \textbf{54.9x }\\
			FAT(Ours) &3/3 &\textbf{4.7} &\textbf{10.0x} & \textbf{21.5}& \textbf{86.7x }\\
			\hline
			\textbf{ResNet-34} \\
			F.P. & 32/32 & 87.2 &1x & 3759.4 & 1x \\
			%LSQ &8/8   & 21.9 & 4.0x & 1920.3 & 2.0x \\
			DSQ &4/4&   11.2  & 7.8x & 69.7  & 53x\\
			FAT(Ours) &4/4& \textbf{11.1}& \textbf{7.9x} & \textbf{64.3} & \textbf{58.5x}\\
			FAT(Ours) &3/3& \textbf{8.6} & \textbf{10.1x }&  \textbf{38.6} & \textbf{97.4x}\\
			\hline
			\textbf{MobileNetV2} \\
			F.P. & 32/32 & 14.1 &1x & 337.9& 1x \\
% 			BOPs for Mobilenetv2 refer to ABS
% https://openreview.net/pdf?id=QjINdYOfq0b
% The BOPs in DJPQ is a confusing value.
% 			DFQ &8/8 & -& - & 355.14 \\
%             RQ& 6/6 & -& -& 355.14 \\
%             UNIQ& 4/8 & -& - & 177.57\\
%             DJPQ &mixed(4-8)  & -& - & 132.28\\
            DQ & 4/8 (m) & 2.3  & 6.1x  & 19.6 & 17.2x \\
            DSQ &4/4 & 2.3 & 6.1x           & 13.2 &25.6x  \\
            APoT &4/4 & 2.3 & 6.1x          & 13.2 &25.6x\\
            % HAQ& 4/8 (m)  &2.4  &5.9x  &10.8 & 31.3x \\
            % ABS &6.1/7.1 &-  &- & 10.9 & 31.0x\\
            % our &5/5 &2.69 &  5.21x & 9.99& 33.83x \\
			FAT(Ours) &4/4 &\textbf{2.1} & \textbf{6.7x} & \textbf{7.4}& \textbf{45.7x}\\
			FAT(Ours) &3/3 &\textbf{2.1} & \textbf{6.7x} &\textbf{ 5.4}& \textbf{62.6x} \\
			\bottomrule
		\end{tabular}
	\end{center}
	\vspace{-0.5cm}
\end{table}

\begin{table}[t]
	\begin{center}
		\caption{Time cost of different quantizers in a 4-bit ResNet-18 on an ARM mobile board, Jetson AGX Xavier. ``LogQ'' and  ``UniQ'' denote logarithmic and uniform quantizers used in our method.}
		\label{table4}
		\begin{tabular}{llll}
		\toprule
		 LQ-Net & APoT   & LogQ (Ours) & UniQ (Ours) \\
		 \hline
		  % & 188 & 223 & 257  & - \\
		  929ms & 857ms   &  800ms & \textbf{398ms} \\
		 \bottomrule
		\end{tabular}
	\end{center}
	\vspace{-0.5cm}
\end{table}
\setlength{\tabcolsep}{1.2pt} 	

\subsubsection{Quantization Shift via Transform}
In Figure~\ref{fig:shift}, we show how much full-precision weight has shifted quantization after transform, i.e., $\rm Q(\mathcal{W}) \neq \rm Q(\rm T(\mathcal{W}))$, where $\mathcal{W}$ is the pretrained full-precision weight, and $\rm T(\cdot)$ is learned transform. From Figure~\ref{fig:shift}, many neurons have shifted quantization results after transform. Weight shifting boosts the information gain during quantization~\cite{qin2020forward}. It verifies that we should not simply assign quantized weights near their full-precision counterparts, due to the difference of solution spaces between full-precision and low-bitwidth models.

\begin{figure}[t]
\centering
\includegraphics[scale=0.45]{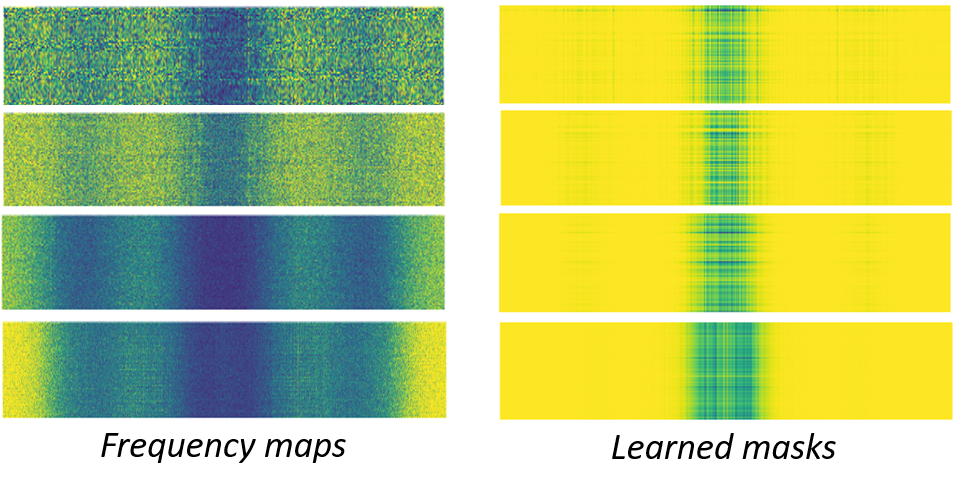}
\caption{Visualization of the frequency maps of convolutional weights and corresponding learned masks in 4 layers of ResNet-34. Warm color means strong spectral density in frequency map and importance learned by mask, respectively. High-frequency part (center) tends to be removed after FAT.}
\label{fig:mask}
\end{figure}

\begin{figure}[t]
\centering
\includegraphics[scale=0.2]{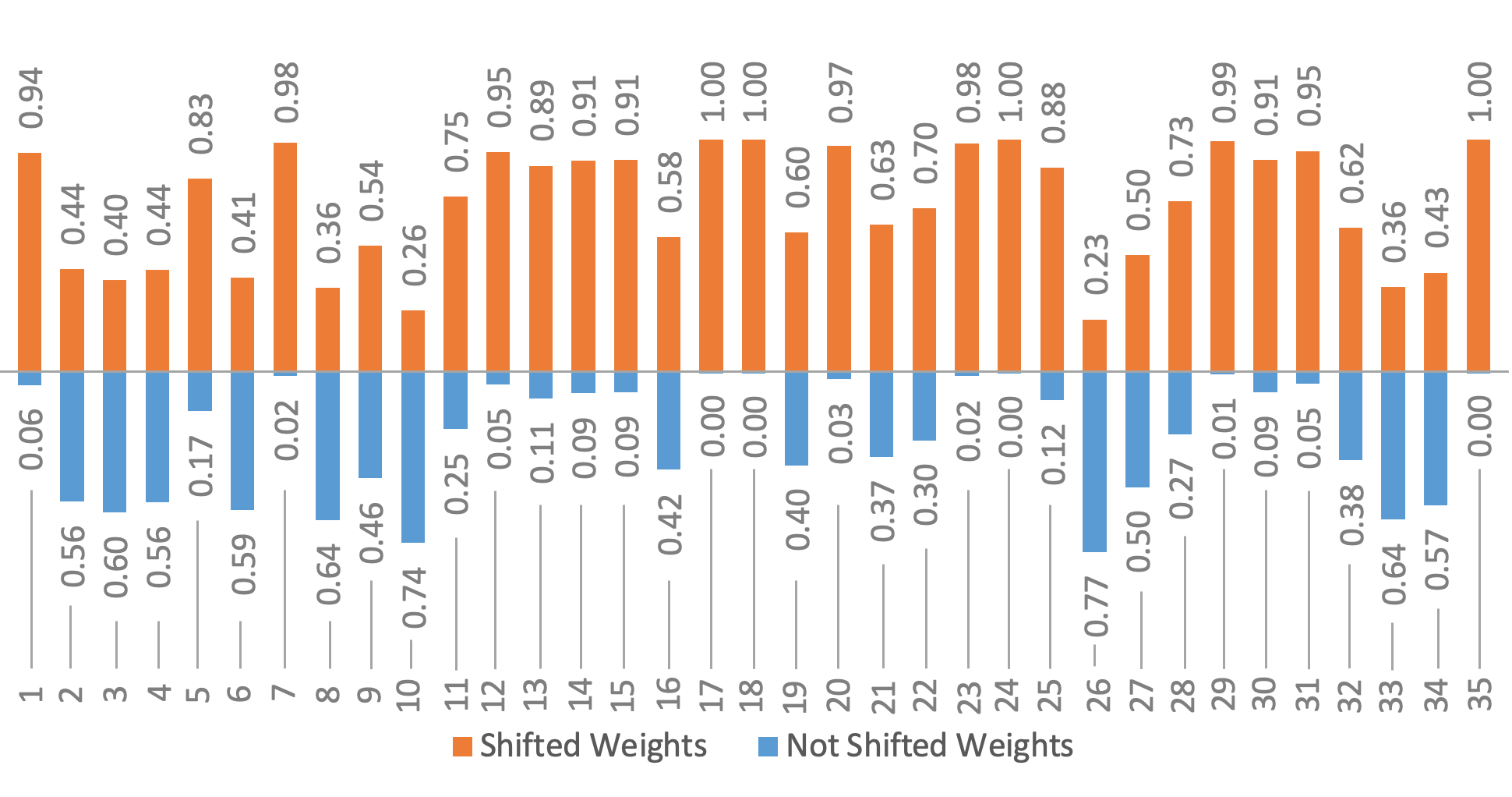}
\caption{Proportions of shifted weights after FAT in various layers in ResNet-34.}
\label{fig:shift}
\end{figure}

\subsubsection{Speed Comparison on Board}
We deploy our method on a mobile device and test the real inference speed of different quantizers. The device we use is Jetson AGX Xavier, a ARM v8.2 64-bit CPU-based architecture. The inference time is reported on the ImageNet dataset with a single thread. From Table~\ref{table4}, we observe that the inference time cost in LQ-Net is relatively large, since different adders and multiplies are employed for adaptive quantization levels. APoT and logarithmic quantizers improve the speed because various quantization levels can be achieved by cheap shift operation. Uniform quantization enjoys high inference efficiency with the smallest time cost, since it employs unified quantizer for all layers during quantization.

\section{Conclusions and Future Work}
This work proposes a novel Frequency-Aware Transformation (FAT) model for low-bitwidth quantization of convolutional neural networks. For the first time, we explicitly define quantization as a trained representation transform in the solution space. The soft mask in the spectral transform learns the importance of weights in each frequency bin. Backed by theoretical analyses, we show the proposed transform enables quantization error reduction and frequency-informative discretization gradient for efficient backpropagation. In addition, the transform introduces negligible extra training cost and zero test cost. Extensive experiments have demonstrated that FAT surpasses existing state-of-the-art methods with higher compression ratios and easy deployment. This work sheds light on learning quantization-friendly representations, instead of designing complicated quantizers to accommodate low-bitwidth models.

Our analyses suggest that a suitable transformation could empower simple quantizers in quantization problems. The transformation is supposed to not only filter trivial components of full-precision data, but also have informative gradient. Future work will explore other transformations on quantization, and extend to other compression strategies like pruning or distillation. The effect of transformations on model robustness is also an open issue.

%% Reference
% \clearpage

{\small
\bibliographystyle{ieee}
\bibliography{egbib}
}

\clearpage
%%%%%%%%%%% Appendix
%% recount the section / equation / theorem etc.
\setcounter{section}{0}
\setcounter{equation}{0}
\setcounter{theorem}{0}
\setcounter{table}{0}
\setcounter{figure}{0}

%\subtitle{Supplementary Materials for FAT: Learning Low-Bitwidth Parametric Representation via Frequency-Aware Transformation}

\twocolumn[
\centering
\title{\LARGE{\textbf{Supplementary Materials for FAT}}}
\vspace{1cm}
]

%%%% MSE of quantization error
\section{Quantization Error of FAT}
% roadmao
\begin{figure*}[tp]
\centering
\includegraphics[scale=0.45]{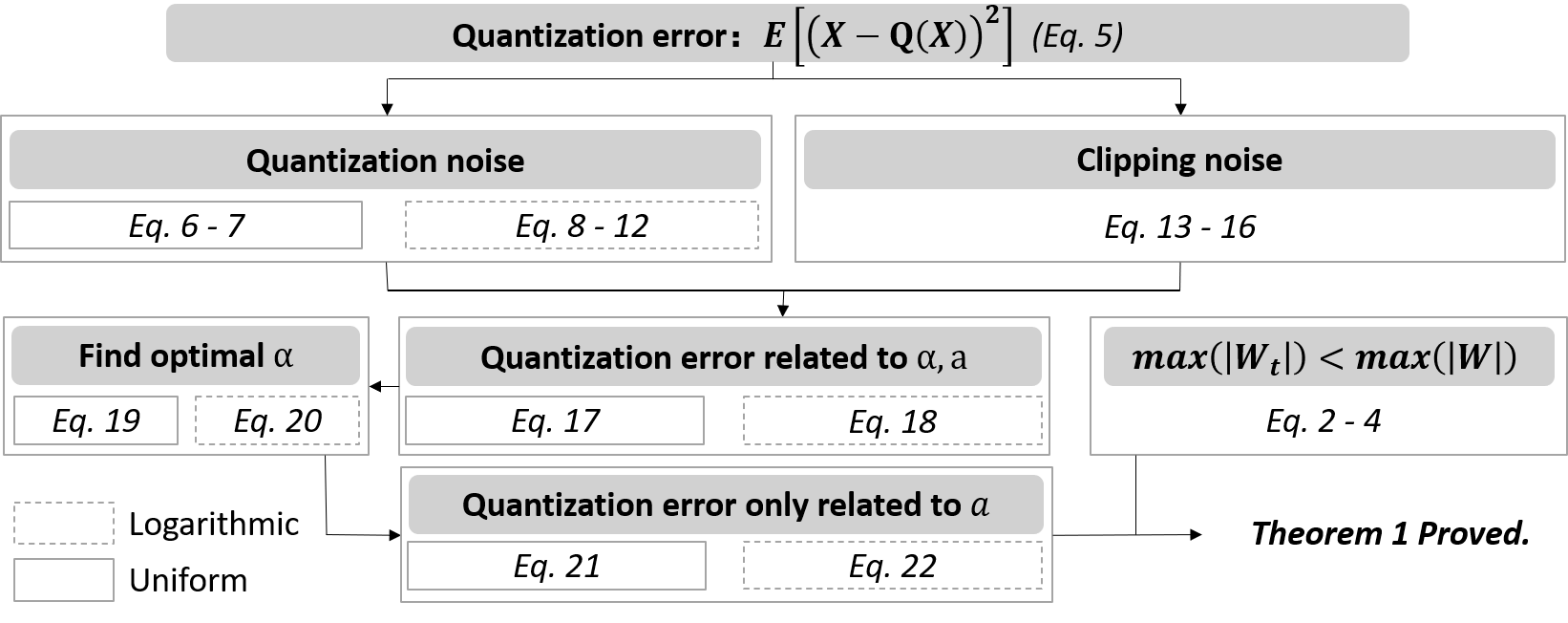}
\caption{Flowchart for the proof of Theorem 1.}
\label{fig:roadmap}
\end{figure*}

\begin{theorem}
\label{theorem:mse_sm}
Assume the given weights before transform $\mathcal{W}$ satisfies ${\rm{Laplace}(0,b)}$ distribution. Then the following two inequalities hold for both uniform and logarithmic quantization:
\begin{subequations}
\begin{equation}
\label{eq:inequal_1}
\max(|\mathcal{W}_t|) \leq \max(|\mathcal{W}|),
\end{equation}
\begin{equation}
\label{eq:inequal_2_sm}
E[(\mathcal{W}_t-{\rm Q}(\mathcal{W}_t))^2] \leq E[(\mathcal{W}-{\rm Q}(\mathcal{W}))^2].
\end{equation}
\end{subequations}
\end{theorem}

% proof
\begin{proof} We first show that the proposed FAT enables to tighten the weight towards zero by deactivating the weight amplitude $a$ from Eq.\ref{eq:dft_mat1} to Eq.\ref{eq:Wt_by_F_Mhat}. Then we prove that the quantization error can be approximated as two parts, quantization noise and clipping noise. The quantization error can be written as a function related with the clipping threshold and weight amplitude.  Clipping threshold is learnable during backpropagation, we assume the clipping threshold can reach its optimal value during training process. Then quantization error is positively correlated with weight amplitude only. Since the amplitude is deactivated,  the quantization error deceases via the proposed FAT. Figure~\ref{fig:roadmap} shows the flow of the proof.

 We use notations $\mathcal{W}$ and $\mathcal{W}_f$ to mean weight vectors and its frequency map, respectively. The process of 1-D Discrete Fourier transform $\mathcal{F}(\cdot)$ can be expressed by matrix multiplication with $F$ as:
\begin{subequations}
\begin{equation}
\label{eq:dft_mat1}
\mathcal{W}_f = F \mathcal{W}
\end{equation}
\begin{small}
\begin{equation}
\label{eq:dft_mat2}
F = \left[ \begin{array}{ccccc}
1 & 1 & 1 & \cdots & 1 \\
1 & \omega^1 & \omega^2 & \cdots & \omega^{(N-1)} \\
1 & \omega^2 & \omega^4 & \cdots & \omega^{2(N-1)}\\
\vdots & \vdots & \vdots & \ddots & \vdots\\
1 & \omega^{(N-1)} & \omega^{2(N-1)} & \cdots & \omega^{(N-1)(N-1)}
\end{array} \right]
\end{equation}%
\end{small}%
\end{subequations}%
where $\omega=e^{-\frac{2\pi\cdot j}{N}}$ with $j^2 = -1$. Therefore, $\mathcal{W}_t$ can be formalized using $F$ as: \begin{equation}
\label{eq:Wf_by_F}
\mathcal{W}_t = F^{-1}(M \odot F \mathcal{W}).
\end{equation}

Since Hadamard product is exchangeable, we have
\begin{equation}
\label{eq:Wt_by_F_Mhat}
\mathcal{W}_t = F^{-1}(F \mathcal{W} \odot M) ,
\end{equation}

Since the mask $M$ is generated with Sigmoid function,  all elements in mask ranges $[0,1]$ during training. Hence,  the proposed FAT helps tighten the original weights $\mathcal{W}$ towards zero. We use amplitude $a_1$ and $a_2$  to denote  $\max (|\mathcal{W}|)$ and $\max(|\mathcal{W}_t|)$, respectively. We have $a_2 \leq a_1$.

\textbf{Quantization Error Formulation.} According to \cite{banner2019post}, the quantization error can be divied into two terms, namely, \textit{quantization noise} and \textit{clipping noise}. Without loss of generality, we use  $X$ to denote the full-precision random variable, instead of using weight notation $\mathcal{W}$. $X$  is assumed to be zero-centering. The whole quantization error is formalized as below:
\begin{equation}
\begin{small}
\label{eq:3terms}
\begin{aligned}
E[(X-{\rm Q}_u(X))^2]&=\int_{-\infty}^{-\alpha}f(x)\cdot(x+\alpha)^2\\
&+\sum_{i=0}^{2^m-1}\int_{-\alpha+i\cdot \Delta}^{-\alpha+(i+1)\cdot \Delta} f(x)\cdot f(x-q_i)^2 dx\\
&+\int^{+\infty}_{\alpha}f(x)\cdot(x-\alpha)^2,
\end{aligned}
\end{small}
\end{equation}
where $\Delta = \frac{2\alpha}{2^m}$ is the approximated intervals between quantization levels under uniform quantizer, and $\alpha, m$ are the clipping value and bitwidth, respectively.  Quantization noise (the second term) considers the error within clipping threshold, and clipping noise (the first term and third term) considers the error outside clipping threshold.

In Eq.~\ref{eq:3terms},  \cite{banner2019post} ignores the actual data amplitude but treat them as $-\infty$ or $+\infty$, which is inconsistent with weight that the extreme values are usually far away from $-\infty$ or $+\infty$. In the following, we derive a more accurate expression of the quantization error by considering data amplitude. And more generally, we extend the quantization error function from uniform quantizer to logarithmic quantizer.

\textbf{Uniform Quantization Noise.}
In Eq.~\ref{eq:3terms}, the second term is the quantization noise. By assuming the density function $f(\cdot)$ is a construction of a piece-wise linear function, the quantization noisy can be approximated as:
\begin{equation}
\label{eq:term2}
\begin{aligned}
&\sum_{i=0}^{2^m-1} \int^{-\alpha + (i+1)\cdot \Delta}_{-\alpha + i\cdot \Delta} f(x) \cdot (x-q_i)^2 dx\\
&\approx\sum_{i=0}^{2^m-1} \int^{-\alpha + (i+1)\cdot \Delta}_{-\alpha + i\cdot \Delta} f(q_i) \cdot (x-q_i)^2 dx\\
&=\frac{2\cdot \alpha^3}{3 \cdot 2^{3m}} \cdot \sum_{i=0}^{2^M-1} f(q_i).
\end{aligned}
\end{equation}

By substituting $f(x) = \frac{1}{2\alpha}$ into Eq.~\ref{eq:term2}, the quantization noise can be further approximated as below:
\begin{equation}
\label{eq:est_term2}
\sum_{i=0}^{2^m-1} \int^{-\alpha + (i+1)\cdot \Delta}_{-\alpha + i\cdot \Delta} f(x) \cdot (x-q_i)^2 dx \approx \frac{\alpha^2}{3\cdot 2^{2m}}.
\end{equation}

% log quantization
\textbf{Logarithmic Quantization Noise.}
The quantization levels of logarithmic quantization are powers-of-two values or zero as below:
% eq
\begin{small}
\begin{equation}
\label{eq:log_quant_levels}
\alpha \times \left\{0, \pm 2^{-2^{m-1}+1}, \pm 2^{-2^{m-1}+2}, \cdots, \pm 2^{-1}, \pm 1 \right\},
\end{equation}
\end{small}
where $\alpha$ is the clipping value and $m$ is the bitwidth.

Since the quantization levels and the weight distribution are both symmetric about the x-axis, we can double the quantization noise on the positive x-axis. To approximate the error, we divide the interval $(0,\alpha]$ into $(0,2^{-2^{m-1}+1}]$ and $(2^{-2^{m-1}+1}, \alpha]$. Therefore, the quantization noise can be formalized as below:

% eq
\begin{small}
\begin{equation}
\label{eq:log_intervals}
\begin{aligned}
&2\cdot \left[\int_0^\alpha f(x)\cdot(x-{\rm Q}(x))^2 dx\right]\\
=&2\cdot \left[ \int_0^{\alpha\cdot2^{-2^{m-1}+1}} f(x)\cdot(x-{\rm Q}(x))^2 dx\right.\\
&\left.+\int_{\alpha \cdot2^{-2^{m-1}+1}}^{\alpha} f(x)\cdot(x-{\rm Q}(x))^2 dx\right]
\end{aligned}
\end{equation}
\end{small}

We assume $f(x) = \frac{1}{2\alpha}$ and all values are rounded to the midpoint of the given interval as well. Eq.~\ref{eq:log_quant_term1} and \ref{eq:log_quant_term2} calculate the quantization noise in $(0,2^{-2^{m-1}+1}]$ and $(2^{-2^{m-1}+1}, \alpha]$, respectively.:
% eq
\begin{small}
\begin{equation}
\label{eq:log_quant_term1}
\begin{aligned}
&2\cdot \int_0^{\alpha\cdot2^{-2^{m-1}+1}}f(x)\cdot(x-{\rm Q}(x))^2 dx= \frac{\alpha^2}{3}\cdot 2^{-3\cdot2^{m-1}+1}
\end{aligned}
\end{equation}
\end{small}

% eq
\begin{small}
\begin{equation}
\label{eq:log_quant_term2}
\begin{aligned}
&2\cdot\int_{\alpha \cdot2^{-2^{m-1}+1}}^{\alpha} f(x)\cdot(x-{\rm Q}(x))^2 dx\\
=& 2\cdot \sum_{i=0}^{2^{m-1}-2}\int_{\alpha\cdot x_i}^{\alpha\cdot x_{i+1}} f(x)\cdot(x-{\rm Q}(x))^2 dx\\
=& 2\cdot \frac{\alpha^3}{2\alpha}\cdot \sum^{2^{m-1}-2}_{i=0} \left[ \frac{1}{3}\cdot(x-\frac{x_i+x_{i+1}}{2})^3\right]_{ x_i}^{x_{i+1}}\\
=& \frac{\alpha^2}{21}\cdot(2^{-2}-2^{-3\cdot2^{m-1}+1}),
\end{aligned}
\end{equation}
\end{small}
where $x_i = 2^{-2^{m-1}+i+1}$, $x_{i+1} = 2^{-2^{m-1}+i+2}$.

By substituting Eq.~\ref{eq:log_quant_term1} and \ref{eq:log_quant_term2} into Eq.~\ref{eq:log_intervals}, the quantization noise of logarithmic quantization is as follows:
% eq
\begin{small}
\begin{equation}
\label{eq:log_quant_error}
\begin{aligned}
&2\cdot \left[\int_0^\alpha f(x)\cdot(x-{\rm Q}(x))^2 dx\right]\\
=& \frac{\alpha^2}{3}\cdot 2^{-3\cdot2^{m-1}+1} + \frac{\alpha^2}{21}\cdot(2^{-2}-2^{-3\cdot2^{m-1}+1})\\
=& \frac{\alpha^2}{84}\cdot\left( 1+ 3\cdot 2^{-3\cdot 2^{m-1}+4}\right).
\end{aligned}
\end{equation}
\end{small}

% clipping noise
\textbf{Clipping Noise.}
The clipping noise considers the error outside the clipping threshold, which do not involves the setting quantization levels. Therefore, clipping noise is equivalent for both uniform quantizer and logarithmic quantizer. By observing Eq.~\ref{eq:3terms}, for symmetrical distributions around zero like ${\rm{Laplace}(0,b)}$ distribution, the first term and third term is equal. Hence, the clipping noise can be written as:
\begin{equation}
\label{eq:term1_3}
2\cdot \int_{\alpha}^{\infty}f(x)\cdot(x-\alpha)^2dx.
\end{equation}
To approximate the clipping noise more accurately, we take the real dynamic range into consideration. We denote the data amplitude $\max(|X|)$ as $a$, therefore, $x\in[-a,a]$ instead of $(-\infty, +\infty)$. By adding the limitation on the range of $x$, Eq.~\ref{eq:term1_3} is reformalulated as:
\begin{equation}
\label{eq:new_term1_3}
2\cdot \int_{\alpha}^{a}f(x)\cdot(x-\alpha)^2dx.
\end{equation}

For ${\rm Laplace}(0,b)$ distribution, its cumulative distribution function can be formalized as: 
\begin{equation}
\label{eq:cdf}
\psi(x) = \frac{e^{-\frac{x}{b}}}{2}\cdot [2\alpha-2b^2-\alpha^2 - x^2 -2(b-a)x].
\end{equation}
By submitting Eq.~\ref{eq:cdf} into Eq.~\ref{eq:new_term1_3}, the clipping noise term can be approximated as a function of $\alpha$ and $a$ as follows:
\begin{small}
\begin{equation}
\begin{aligned}
\label{eq:approx_term1_3}
2 &\cdot \int_{\alpha}^{a}f(x)\cdot(x-\alpha)^2dx = \psi(a) - \psi(\alpha)\\
&= e^{-\frac{a}{b}}\cdot [2\alpha b-2b^2-\alpha^2 - a^2 -2(b-\alpha)a]\\
&-e^{-\frac{\alpha}{b}}\cdot [2\alpha b-2b^2-\alpha^2 - \alpha^2 -2(b-\alpha)\alpha]\\
&=e^{-\frac{a}{b}}\cdot[2\alpha b + 2\alpha a - b^2 - \alpha^2 - (b+a)^2] + e^{-\frac{\alpha}{b}}\cdot 2b^2\\
&= e^{-\frac{a}{b}}\cdot [2\alpha a - (b+a)^2 - (b-\alpha)^2] + b^2 \cdot e^{-\frac{\alpha}{b}}.
\end{aligned}
\end{equation}
\end{small}

\textbf{Quantization Error related to $\alpha,a$.}
By gathering Eq.~\ref{eq:est_term2} and \ref{eq:approx_term1_3}, we get the new expression of the \textit{uniform} quantization error as a function of $\alpha$ and $a$ as below:
\begin{small}
\begin{equation}
\label{eq:new_error}
\begin{aligned}
f_{u}(\alpha,a) &= e^{-\frac{a}{b}}\cdot [2\alpha a - (b+a)^2 - (b-\alpha)^2]\\
&+ b^2 \cdot e^{-\frac{\alpha}{b}} + \frac{\alpha^2}{3\cdot 2^{2m}}.
\end{aligned}
\end{equation}
\end{small}

Putting Eq.~\ref{eq:log_quant_error} and \ref{eq:approx_term1_3} together, the logarithmic quantization error can be formalized as:
% eq
\begin{small}
\begin{equation}
\label{eq:log_mse}
\begin{aligned}
f_{log}(\alpha,a) &= e^{-\frac{a}{b}}\cdot [2\alpha a - (b+a)^2 - (b-\alpha)^2]\\
&+ b^2 \cdot e^{-\frac{\alpha}{b}} + \frac{\alpha^2}{84}\cdot(1+3\cdot 2^{-3\cdot 2^{m-1}+1}).
\end{aligned}
\end{equation}
\end{small}

Both $f_u$ and $f_{log}$ are functions of $\alpha$ and $a$.

\textbf{Finding the Optimal $\alpha$.}
Clipping value $\alpha$ is a trainable parameter during backpropagation. Ideally, we can find an optimal $\alpha^*$ to minimize quantization error function $f(\alpha, a)$. In order to make a fair comparison between $E[(\mathcal{W}_t-{\rm Q}(\mathcal{W}_t))]$ and $E[(\mathcal{W}-{\rm Q}(\mathcal{W}))^2]$, we treat $f(\alpha, a)$ as a function only related to $a$ by finding optimal $\alpha^*$ for every selected $a$.

For uniform quantization, the optimal $\alpha_u^*$ can be found by solving the following equation:
\begin{small}
\begin{equation}
\label{eq:est_alpha_opt}
\begin{aligned}
\left. \frac{\partial f_u(\alpha,a)}{\partial \alpha}\right|_{\alpha = \alpha_u^*} =&\frac{2 \alpha_u^*}{3\cdot2^{2m}}+e^{-\frac{a}{b}}\cdot(a+b-\alpha_u^*)\\
&-b\cdot e^{-\frac{\alpha^*}{b}}\\
=&0.
\end{aligned}
\end{equation}
\end{small}

Similarly, the optimal $\alpha_{log}^*$ is the solution of the equation below:
\begin{small}
\begin{equation}
\label{eq:est_alpha_opt2}
\begin{aligned}
\left. \frac{\partial f_{log}(\alpha,a)}{\partial \alpha} \right|_{\alpha=\alpha_{log}^*} =& \frac{\alpha}{41}\cdot(1+3\cdot2^{-3\cdot2^{m-1}+4})-b\cdot e^{-\frac{\alpha_{log}^*}{b}}\\
&+e^{-\frac{a}{b}}\cdot(a+b-\alpha_{log}^*)\\
=&0.
\end{aligned}%
\end{equation}%
\end{small}%
Though it is hard to find analytical solutions, we can always search for the numerical solutions.

\textbf{Quantization Error only related to $a$.} By taking Eq.~\ref{eq:est_alpha_opt} into Eq.~\ref{eq:new_error} and Eq.~\ref{eq:est_alpha_opt2} into Eq.~\ref{eq:log_mse}, we can get the expressions of quantization errors related to $a$ only. Eq~\ref{eq:mse_1} and  shows the uniform and logarithmic cases, respectively:
\begin{small}
\begin{equation}
\label{eq:mse_1}
\begin{aligned}
g_u(a)=&  e^{-\frac{a}{b}}\cdot [2\alpha_u^* a - (b+a)^2 - (b-\alpha_u^*)^2]\\
&+ b^2 \cdot e^{-\frac{\alpha_u^*}{b}} + \frac{{\alpha_u^*}^2}{3\cdot 2^{2M}},
\end{aligned}
\end{equation}

\begin{equation}
\label{eq:log_mse_a}
\begin{aligned}
g_{log}(\alpha) =& e^{-\frac{a}{b}}\cdot [2\alpha_{log}^* a - (b+a)^2 - (b-\alpha_{log}^*)^2]\\
&+ b^2 \cdot e^{-\frac{\alpha_{log}^*}{b}} + \frac{\alpha_{log}^2}{84}\cdot(1+3\cdot 2^{-3\cdot 2^{m-1}+1}).
\end{aligned}
\end{equation}
\end{small}

Figure~\ref{fig:eq13a} visualizes the two curves $g_u$ and $g_{log}$, which shows quantization errors keep rising when $a$ increases in both two cases. Because $a_2=\max (|\mathcal{W}_t|) < a_1=\max (|\mathcal{W}|)$, we can conclude that: $E[(\mathcal{W}_t-{\rm Q}(\mathcal{W}_t))^2] \leq E[(\mathcal{W}-{\rm Q}(\mathcal{W}))^2]$.

% fig
\begin{figure}
\centering
\includegraphics[scale = 0.8]{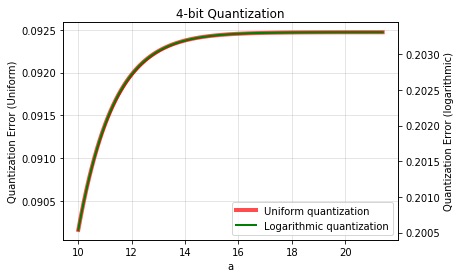}
\caption{With optimal clipping value $\alpha^*$ for every different amplitudes $a$, the curves show the quantization error goes up as amplitude increases for both uniform and logarithmic quantization.}
\label{fig:eq13a}
\end{figure}
\end{proof}

%%%% Gradients
\section{Informative Discretization Gradient}
In the following we derive the the gradients of transformed weights $\mathcal{W}_t$ to the mask $M$ and the original weights $\mathcal{W}$ during the backward propagation, respectively. To show the chain rule clearly and avoid the redundancy of symbols, we employ some new notations to denote the intermediate variables:
%% intermediate variables
\begin{subequations}
\begin{equation}
t_1 = \mathcal{F}(\mathcal{W}(i,:)),
\end{equation}
\begin{equation}
t_2 = M(i,:) \odot t_1.
\end{equation}
\end{subequations}
Then $\mathcal{W}_t$ can be represented as $\rm \mathcal{F}^{-1}(t_2)$.

\subsection{Gradient of $\mathcal{W}_t$ to $M$}
%% Theorem 1
\begin{theorem}
\label{theorem:w_de_M}
The gradient of $\mathcal{W}_{t}$ to $M$ during backward propagation is $\frac{\partial \mathcal{W}_{t}}{\partial M} = \sum_{k_1=0}^{N-1} \sum_{i,k_2}\frac{\partial \mathcal{W}_{t}(i,k_1)}{\partial M(i,k_2)}$, with 

% eq: w_de/M element-wise
\begin{small}
\begin{equation}
\label{eq:grad_wde_M}
\frac{\partial \mathcal{W}_{t}(i,k_1)}{\partial M(i,k_2)} = \frac{1}{N} \sum_{n=0}^{N-1} \mathcal{W}(i,n) \cdot \cos{\left(\frac{2\pi (k_1 - n) k_2}{N}\right)},
\end{equation}
\end{small}
where $i=0,\cdots,C_{out}-1$ and $k_1,k_2=0,\cdots,N-1.$
\end{theorem}

%% Prove theorem 1
\begin{proof}
In the main paper, Figure 4 illustrates how to get the gradient matrix from the 3-D gradient tensor. Therefore, here we focus on how to compute the gradient entries.

According the chain rule, the gradient $\frac{\partial \mathcal{W}_{t}(i,k_1)}{\partial M(i,k_2)}$ can be calculated in the following flow:
\begin{equation}
\label{eq:th1_chain}
\frac{\partial \mathcal{W}_t(i,k_1)}{\partial M(i,k_2)} = \frac{\partial \mathcal{W}_t(i,k_1)}{\partial t_2(k1)} \cdot \frac{\partial t_2(k_1)}{\partial M(i,k_2)}.
\end{equation}
It is worth noting that since $t_2$ is obtained by $M(i,:)$ and $t_1$ using element-wise product, $M(i,k_2)$ is only related to $t_2(k_2)$, and has nothing to do with $t_2(n)$ where $n \neq k_2$.

The two terms on the right side of Eq.~\ref{eq:th1_chain} are easy to compute, whose results are:
\begin{subequations}
\begin{equation}
\label{eq:th1_1}
\frac{\partial t_2(k_1)}{\partial M(i, k_2)} = \sum_{n=0}^{N-1}\mathcal{W}(i,n)\cdot e^{-j\cdot \frac{2\pi}{N}\cdot k_2 \cdot n},
\end{equation}
\begin{equation}
\label{eq:th1_2}
\frac{\partial \mathcal{W}_t(i,k_1)}{\partial t_2(k_1)} = \frac{1}{N} \cdot e^{j\cdot \frac{2\pi}{N} \cdot k_1 \cdot k_2}.
\end{equation}
\end{subequations}

By substituting Eq.~\ref{eq:th1_1} and \ref{eq:th1_2} into Eq.~\ref{eq:th1_chain} and taking the real part, we can get Eq.~\ref{eq:grad_wde_M}.
\end{proof} 

% roadmap2
\begin{figure}
\centering
\includegraphics[scale=0.5]{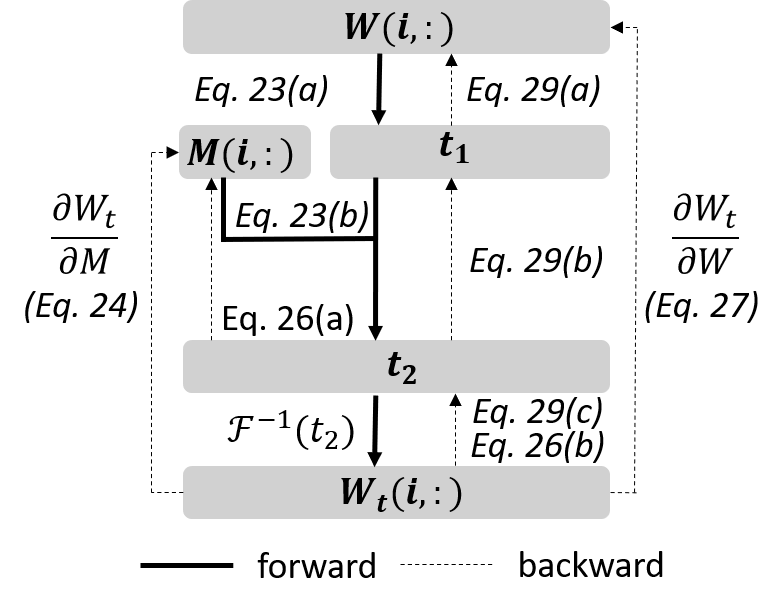}
\caption{Flowchart for computing $\frac{\partial \mathcal{W}_t}{\partial M}$ and $\frac{\partial \mathcal{W}_t}{\partial \mathcal{W}}$.}
\label{fig:my_label}
\end{figure}

\subsection{Gradient of $\mathcal{W}_t$ to $\mathcal{W}$}
%% Theorem 2
\begin{theorem} 
\label{theorem:w_de_w}
Given a convolutional filter with index $i$, the gradient of $\mathcal{W}_{t}$ to $\mathcal{W}$ during backpropagation is a symmetric matrix that accumulates mask effect from all frequency bases,

% eq: w_de/w element-wise
\begin{small}
\begin{equation}
\label{eq:grad_wde_w_sm}
\frac{\partial \mathcal{W}_{t}(i,k_1)}{\partial \mathcal{W}(i,k_2)} = \frac{1}{N}\sum_{n=0}^{N-1} M(i,n) \cdot \cos{\left( \frac{2\pi (k_1 - k_2) n}{N}\right)},
\end{equation}
\end{small}
where $i=0,\cdots,C_{out}-1$ and $k_1,k_2=0,\cdots,N-1.$
\end{theorem} 

%% Prove theorem 2
\begin{proof}
According to the chain rule, the gradient $\frac{\partial \mathcal{W}_{t}(i,k_1)}{\partial \mathcal{W}(i,k_2)}$ can be represented as 
\begin{small}
\begin{equation}
\label{eq:th2_chain_rule}
\frac{\partial \mathcal{W}_t (i,k_1)}{\partial \mathcal{W}(i, k_2)} = \sum_{n = 0}^{N-1}\frac{\partial \mathcal{W}_t(i,k_1)}{\partial t_2(n)} \cdot \frac{\partial t_2(n)}{\partial t_1(n)} \cdot \frac{\partial t_1(n)}{\mathcal{W}(i,k_1)}.%
\end{equation}%
\end{small}%
It is worth noting that every element in $t_2$ contains the information from $\mathcal{W}(i,k_2)$ due to the Fourier Transform. The expressions of the three terms at the right end of Eq.~\ref{eq:th2_chain_rule} are:
\begin{subequations}
\begin{equation}
\label{eq:th2_1}
\frac{\partial t_1(n)}{\partial \mathcal{W}(i,k_1)} = e^{-j\cdot \frac{2\pi}{N} \cdot n k_1},
\end{equation}
\begin{equation}
\label{eq:th2_2}
\frac{\partial t_2(n)}{t_1(n)} = M(i,n),
\end{equation}
\begin{equation}
\label{eq:th2_3}
\frac{\partial \mathcal{W}_t(i, k_1)}{t_2(n)} = \frac{1}{N}\cdot e^{j\cdot \frac{2\pi}{N}\cdot k_1 \cdot n}.
\end{equation}
\end{subequations}

By substituting Eq.~\ref{eq:th2_1} to \ref{eq:th2_3} to Eq.~\ref{eq:th2_chain_rule}, and discarding the imaginary part of the results, we get Eq.~\ref{eq:grad_wde_w_sm}.

In the following, we prove that if the elements in $M$ are all set as $1$, the transform is an identity mapping $\mathcal{W}_t =  F^{-1}(M \odot F\mathcal{W}) = F^{-1}F\mathcal{W}=\mathcal{W}$:
% eq
\begin{equation}
\begin{small}
\begin{aligned}
\frac{\partial \mathcal{W}_{t}(i,k_1)}{\partial \mathcal{W}(i,k_2)} =& \frac{1}{N}\sum_{n=0}^{N-1} \cos{\left( \frac{2\pi (k_1 - k_2) n}{N}\right)}\\
=& \begin{cases}
\frac{1}{N}\sum_{n=0}^{N-1} \cos{(0)}, k_1 = k_2\\
\frac{1}{N}\cdot {\rm real}\left(\frac{1-e^{j\cdot\frac{2\pi}{N}(k_1-k_2)\cdot N}}{1-e^{j\cdot\frac{2\pi}{N}\cdot(k_1-k_2)}}\right), k_1 \neq k_2
\end{cases}\\
=& \begin{cases}
\frac{1}{N}\sum_{n=0}^{N-1} 1, k_1 = k_2\\
\frac{1}{N}\cdot {\rm real}\left(\frac{1-1}{1-e^{j\cdot\frac{2\pi}{N}\cdot(k_1-k_2)}}\right), k_1 \neq k_2
\end{cases}\\
=& \begin{cases}
1, k_1 = k_2\\
0, k_1 \neq k_2
\end{cases}.
\end{aligned}
\end{small}
\end{equation}
In this special case, the mask allows weights in all frequencies flow to the quantized model by 100\% percent. Without filtering trivial weight component, the gradient matrix degenerates to an identity matrix and the discretization gradient degenerates to the STE, which is used in previous methods.

\end{proof} 

%% MSE analysis
\section{Effectiveness of Mask}
\label{sec:MSE_compare}
In this section, we verify that the proposed mask in FAT is able to remove the quantization-useless information in weight, instead of  constraining the amplitude of weight simply. Multiplying a small coefficient $\beta \in (0,1)$ on weight is a vanilla way to constrain the amplitude of weight. By doing so, the estimated standard derivation decreases linearly under Laplace distribution. The evaluate metric is defined as:
\begin{equation}
\begin{small}
\label{eq:MSE_compare}
MSE = MSE(\mathcal{W},\mathcal{W}_{t}) + MSE(\mathcal{W}_t, {\rm Q}(\mathcal{W}_{t})),%
\end{small}%
\end{equation}%
which evaluates the distance between 1) original weight and transformed weight, 2) transformed weight and quantized weight. We compare three transforms applied on weight, $W_t = F^{-1}(M \odot F \mathcal{W})$, $W_t = 0.5 \mathcal{W}$ and $W_t = 0.75 \mathcal{W}$. Table~\ref{tab:MSE_compare} shows the results on a 4-bit ResNet34. We can observe that the MSE with our proposed transform is steadily small across different layers.
% tab
\begin{table}[t]
\scriptsize
\centering
\setlength{\tabcolsep}{2mm}{
\begin{tabular}{cccc}
\toprule
No. Layer &   $MSE (\beta=0.5)$ & $MSE (\beta=0.75)$ & $MSE$ (Ours)\\
\midrule
0  &  0.54 & 0.21  & \textbf{0.11} \\
5  & 0.51 & 0.17 &\textbf{0.07}\\
10  & 0.52 & 0.17  & \textbf{0.07}\\
15  & 0.51 & 0.18 & \textbf{0.09} \\
20  & 0.51 & 0.16 & \textbf{0.04}\\
25 & 0.54 & 0.20 & \textbf{0.09}\\
30  & 0.70 & 0.34 & \textbf{0.30}\\
\bottomrule
\end{tabular}
}
\caption{Error comparison of different transforms to tighten the weight. Simply reducing the amplitude by deceasing standard deviation cannot reduce quantization error and keep information of original weight simultaneously. Instead, the  mask in FAT constrains the amplitude by removing the redundant frequency components, which not only reduces quantization error, but also preserves the dominant information in original weight.}
\label{tab:MSE_compare}
\end{table}

\section{Categorization of Quantization Methods}
In this section, we give a categorization that shows different approaches employed by the recent state-of-the-art methods, including mixed-precision, adaptive or self-defined quantization level and learnable training policy to train quantizers. The commonly used learnable training policy involves noise injection, reinforcement learning or weight pruning. Mixed-precision or adaptive quantization level increases the  difficulty of deployment, and learnable training policy makes the method not very easy to implement during training process.

As shown in Table~\ref{table1_sm}, we compare all methods appeared in the main paper, including WAGE \cite{wu2018training}, LQ-Net \cite{zhang2018lq}, PACT\cite{choi2018pact}, RQ \cite{louizos2018relaxed}, UNIQ \cite{baskin2018uniq}, DQ \cite{tung2018deep}, BCGD \cite{baskin2018nice} \cite{tung2018deep}, DSQ \cite{gong2019differentiable}, QIL \cite{jung2019learning}, HAQ \cite{wang2019haq},  APoT \cite{li2020additive}, HMQ \cite{habi2020hmq} DJPQ \cite{wang2020differentiable}, LSQ \cite{esser2020learned}.

Rather than learning complicated quantizers to fit the pretrained full-precision values,  the proposed FAT attempts to transform the weights to quantization-friendly representation. Therefore, FAT enjoys efficiency during both training process and deployment process.

\begin{table}[ht]
\scriptsize
	\begin{center}
 	\setlength{\tabcolsep}{3mm} 
	\caption{Categorization of quantization methods using different approaches, involving mixed-precision, adaptive or self-defined quantization level or learnable training policy. ``Q. Level'' denotes adaptive non-uniform quantization levels.}
		\label{table1_sm}
		\begin{tabular}{l|ccc}
		\toprule
		Methods & Mixed-precision & Q. Level & Learnable Training \\
		\hline
		WAGE &$\times$  & $\times$ &$\surd$ \\
		LQ-Net&$\surd$ &$\surd$& $\times$  \\
		PACT &$\times$ &$\surd$& $\times$  \\
		RQ &$\times$ &$\surd$ & $\surd$  \\
		UNIQ &$\times$ &$\times$ & $\surd$  \\
		DQ &$\times$ &$\times$ & $\surd$  \\
		BCGD &$\times$ &$\times$ & $\surd$  \\
		DSQ   &$\times$ & $\times$ & $\surd$ \\
		QIL & $\times$ & $\surd$ & $\times$ \\
        HAQ & $\surd$ & $\surd$ & $\surd$ \\
        APoT & $\times$ & $\surd$ & $\times$ \\
        HMQ &$\surd$ & $\times$ &  $\surd$ \\
        DJPQ &$\surd$ & $\times$ & $\surd$ \\
        LSQ & $\times$ & $\surd$ & $\times$ \\
        \textbf{FAT (Ours)} &  $\times$  & $\times$  & $\times$ \\
        \bottomrule
		\end{tabular}
	\end{center}
\end{table}
\section{Training Details}
We training our model from the pretrained weights. The bias in convolutional layers are removed for parameters reduction. Before transforming, we normalize the weights to stabilize training. In the CIFAR-10 dataset, the training epoch is set as 300. Batch size is set as 256. We use stochastic gradient descent (SGD) as the optimizer with an initial learning rate of 0.04. The learning rate decreases by scaling 0.1 after 150 and 225 epochs. The weight decay is employed with coefficient 5e-4.  In the ImageNet dataset, we train 120 epochs. The initial learning rate is set as 0.01 and then decay by scaling 0.1 after 30, 60, and 90 epochs, respectively. We adopt weight decay with the coefficient 1e-4. 

\section{Experiments on Logarithmic Quantizer}
In the main paper, we compare our FAT applied on uniform quantizer, with state-of-the-art quantization methods. In this section, we provide experimental results applied on logarithmic quantizer. As shown in Table \ref{table2_sm} and Table \ref{table3_sm}, we can observe that, on both ImageNet and CIFAR-10 datasets, the proposed FAT applied on logarithmic quantizer is able to achieve comparable performance with full-precision model on various network architectures. Even in 2-bit setting for both weight and activation, FAT competes against full-precision model with acceptable performance drop.
% Logarithmic quantizer works better in 2-bit, 3-bit setting
\begin{table}[ht]
\scriptsize
	\begin{center}
 	\setlength{\tabcolsep}{3.5mm} 
	\caption{Performance of FAT applied on logarithmic quantizer with different bitwidths for weight (W) ad activation (A). Acc@1 and  Acc@5 denote top-1 and top-5 accuracy in percentage.}
		\label{table2_sm}
		\begin{tabular}{l|l|lll}
% 			\hline\noalign{\smallskip}
			\toprule
			Architecture & Methods & W/A  & Acc@1 & Acc@5 \\
			\midrule
			\multirow{4}{*}{ResNet-18}& Full-precision & 32/32 & 70.2 & 89.4\\
			~&FAT(Ours) &4/4 &{68.8} & {88.6} \\
			~&FAT(Ours) &3/3 & 68.7 & 88.3 \\
			~&FAT(Ours) &2/2 & 64.3 & 85.5 \\
% 			reruning on the umec
            \hline
			\multirow{4}{*}{ResNet-34}&Full-precision & 32/32 & 73.7 & 91.3 \\ 
            ~&FAT(Ours) & 4/4  & 73.3 & 91.2  \\
            ~&FAT(Ours) &3/3 & 73.2 & 91.2 \\
            ~&FAT(Ours) &2/2 & 70.1 & 89.4 \\
            \bottomrule
		\end{tabular}
	\end{center}
\end{table}

\begin{table}[ht]
\scriptsize
    \begin{center}
        \setlength{\tabcolsep}{3.5mm} 
		\caption{Performance of FAT applied on logarithmic quantizer on CIFAR-10 dataset with different bitwidths.}
        \label{table3_sm}
        \small{
		\begin{tabular}{l|l|ll}
			\toprule
			Architecture & Methods & W/A  & Accuracy \\
			\midrule
			\multirow{4}{*}{VGG-Small} &Full-precision & 32/32& 93.1 \\
% 			FAT(Ours) & 5/5  & 94.12 (94.19)\\
            ~&FAT(Ours) & 4/4  &{93.8}\\
            ~&FAT(Ours)  & 3/3  & {93.7}	\\
            ~&FAT(Ours)  & 2/2  &{93.3}	\\
            \hline
			\multirow{4}{*}{ResNet-20} &Full-precision & 32/32 & 91.6 \\ 
            % FAT(Ours) & 5/5  & 93.03 (91.89)\\
			~&FAT(Ours)  & 4/4  & {92.0}	\\	
			~&FAT(Ours)  & 3/3  & {92.2}	\\	
			~&FAT(Ours)  & 2/2  & {91.1}	\\
	        \hline
			\multirow{4}{*}{ResNet-56} &Full-precision & 32/32 &  93.2 \\
            ~&FAT(Ours) & 4/4 & {93.8}\\
			~&FAT(Ours)  & 3/3  & {93.7}	\\
			~&FAT(Ours)  & 2/2  & {93.3}	\\
            \bottomrule
		\end{tabular}
		}   
	\end{center}
\end{table}

\section{Importance of Weights in the Frequency Domain}
\begin{figure}
\centering
\includegraphics[scale = 0.5]{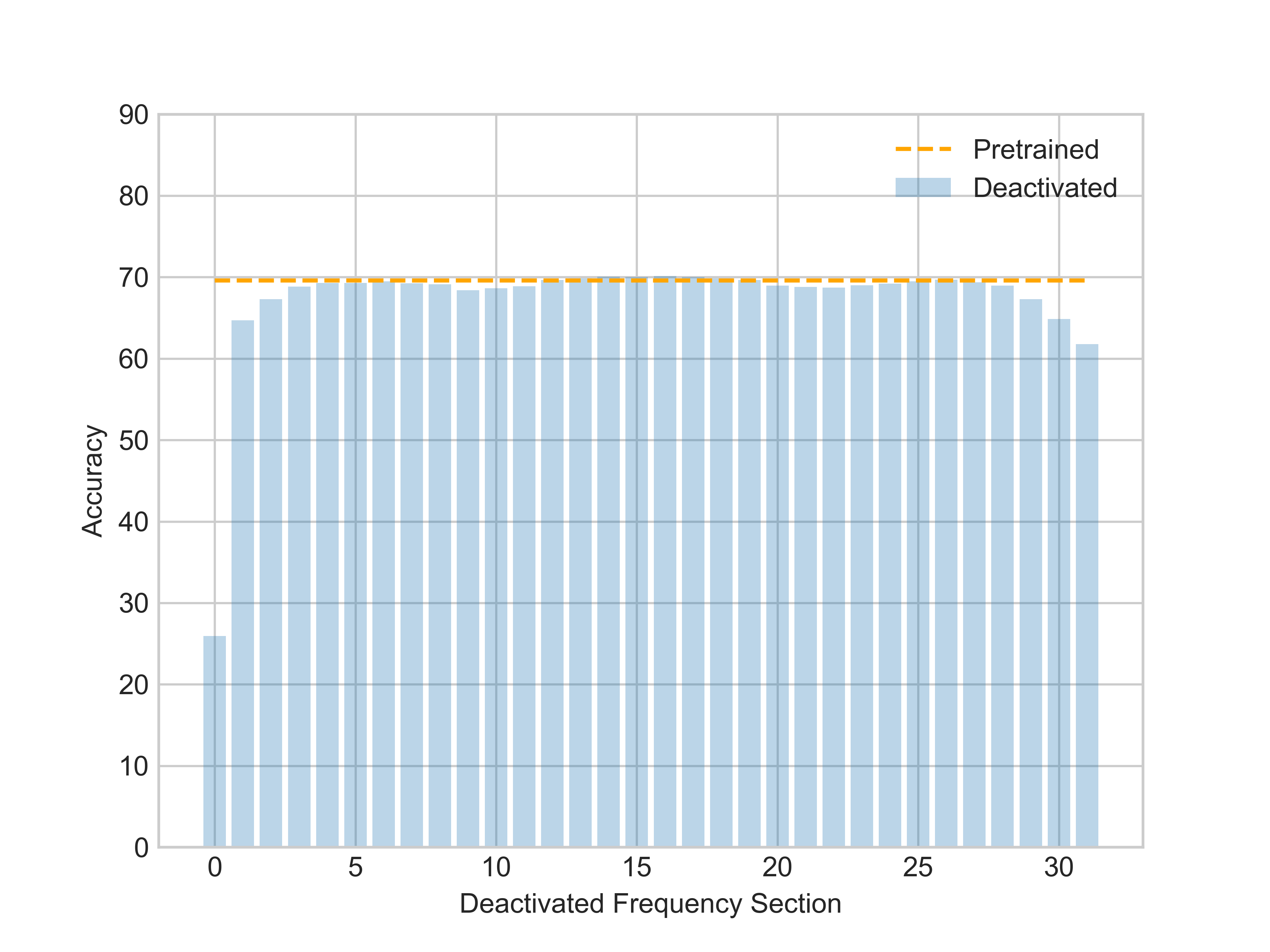}
\caption{Effect of deactivating frequencies in different sections on accuracy. We report the accuracy of pretained ResNet-18 on ImageNet dataset without training. Informative weight components are concentrated on the low-frequency zones. On the other hand, deactivating weights in high frequencies does not hurt the performance. Therefore, the weights on high-frequency component tends to be redundant, which do not need to flow to a capacity-limited quantized model.}
\label{fig:freq}
\end{figure}

In this section, we examine the effect of weights on different frequencies further. For all layers, we map the pretrained weights from spatial domain to frequency domain via Fourier Transform. Then, we divide all frequencies into 32 parts from low to high \textit{(high frequency is shifted to center)}. Instead of learning the mask via gradient descent, we softly mask one frequency part each time by halving the spectral density in this frequency part. By doing so, the weights are deactivated in the selected frequency part, while being kept in other frequency part.

As shown in Figure~\ref{fig:freq},  the performance curve drops rapidly when low frequencies are deactivated. Since the capacity of low-bitwidth model is limited, each bit should be taken full advantage of information extraction during quantization. From the frequency-domain perspective, the proposed spectral transform learns to keep informative frequencies (low frequencies) and mask off trivial frequencies  (high frequencies) regardless of quantizers, thereby achieving competitive performance.
\end{document}